%% file: main.tex
\documentclass{article}
\input{macros}

\title{Understanding the Power and Limitations of Teaching with Imperfect Knowledge}
\author{
Rati Devidze$^{1}$\thanks{Authors contributed equally to this work.}
\and
Farnam Mansouri$^{1}$\footnotemark[1]
\and
Luis Haug$^{2}$
\and
Yuxin Chen$^{3}$
\And
Adish Singla$^{1}$%\thanks{corresponding author}
\affiliations
$^1$Max Planck Institute for Software Systems (MPI-SWS)\\
$^2$ETH Zurich\\
$^3$University of Chicago\\
\emails
\{rdevidze,mfarnam,adishs\}@mpi-sws.org,
lhaug@inf.ethz.ch,
chenyuxin@uchicago.edu
}

\begin{document}
\maketitle

%%%%%%%%%%%%%%%%%%%%%%%%%%%%%%%%%%%%%%%%%%%%%%%%%%%%%%%%
\newtoggle{longversion}
\settoggle{longversion}{true}

%%%%%%%%%%%%%%%%%%%%%%%%%%%%%%%%%%%%%%%%%%%%%%%%%%%%%%%%%%
\input{0_abstract}
\input{1_introduction}
\input{2_relatedwork}
\input{3_statement}
\input{4_noise-in-learning-parameters}
\input{5_noise-in-representation}
\input{6_experiments}
\input{7_conclusion}

%%%%%%%%%%%%%%%%%%%%%%%%%%%%%%%%%%%%%%%%%%%%%%%%%%%%%%%%%%
%\clearpage
%\iftoggle{longversion}
%{
%%\fontsize{9.4pt}{9.5pt}
%%\fontsize{9.5pt}{9.8pt}
%%\selectfont
%}
%{
%%\fontsize{9.pt}{9.8pt}
%%\selectfont
%}

\clearpage
\bibliography{main}
\bibliographystyle{named}

%%%%%%%%%%%%%%%%%%%%%%%%%%%%%%%%%%%%%%%%%%%%%%%%%%%%%%%%%%
\iftoggle{longversion}{
%\clearpage
\onecolumn
\appendix
{\allowdisplaybreaks
\input{8.1.1_appendix_noise-in-learning-parameters_prior}
\input{8.1.2_appendix_noise-in-learning-parameters_learningrate}

\input{8.2.1_appendix_noise-in-representation-sampleX}
\input{8.2.2_appendix_noise-in-representation-embedding}
}
}
{}

%%%%%%%%%%%%%%%%%%%%%%%%%%%%%%%%%%%%%%%%%%%%%%%%%%%%%%%%%%
\end{document}

%% file: macros.tex
\usepackage{etoolbox}

\usepackage{ijcai20}

\usepackage{times}
\usepackage{soul}
\usepackage[normalem]{ulem}
\usepackage{url}
\usepackage[utf8]{inputenc}
\usepackage[small]{caption}
\usepackage{graphicx}
\usepackage{amsmath}
\usepackage{booktabs}
\usepackage{cancel}
\usepackage{algorithm}
\usepackage{algcompatible}
\usepackage[noend]{algpseudocode}

\usepackage[utf8]{inputenc} % allow utf-8 input
\usepackage[T1]{fontenc}    % use 8-bit T1 fonts
\usepackage{url}            % simple URL typesetting
\usepackage{booktabs}       % professional-quality tables
\usepackage{amsfonts}       % blackboard math symbols
\usepackage{nicefrac}       % compact symbols for 1/2, etc.
\usepackage{microtype}      % microtypography

\newcommand{\secref}[1]{Section~\ref{#1}}

\usepackage{float}
\usepackage{times}
\usepackage{mathtools}
\usepackage{multirow,color,graphicx}
\usepackage{subcaption}
\usepackage{caption}
\usepackage{xfrac,amsfonts,amsbsy}
\usepackage[titletoc]{appendix}
\usepackage{xspace}
\usepackage{savesym,verbatim}
\usepackage[titletoc]{appendix}
\usepackage{paralist}
\usepackage{subscript}
\usepackage{tikz}
\usepackage{verbatim}
\usepackage{amsthm}

\usepackage{bbm}

\newtheorem{lemma}{Lemma}
\newtheorem{definition}{Definition}
\newtheorem{theorem}{Theorem}

\makeatletter

\newcommand{\Rmnum}[1]{\expandafter\@slowromancap\romannumeral #1@}
\makeatother

\newcommand{\citet}[1]{\citeauthor{#1} (\citeyear{#1})}
\DeclareMathOperator*{\argmin}{arg\,min}

\newcommand{\R}{\mathbb{R}}

\newcommand{\opt}{\textsc{Opt}}
\newcommand{\teachOpt}[1]{\widetilde{\textsc{Opt}}_{#1}}
\newcommand{\apx}{\widetilde{\textsc{Opt}}}
\newcommand{\Error}{\textsc{Err}}
\newcommand{\error}{\textnormal{err}}

\newcommand{\Hypotheses}{\mathcal{H}}

\newcommand{\hypothesis}{h}
\newcommand{\hstar}{\hypothesis^*}

\newcommand{\Examples}{\mathcal{Z}}
\newcommand{\Instances}{\mathcal{X}}

\newcommand{\examples}{Z}

\newcommand{\example}{{z}}
\newcommand{\instance}{{x}}
\newcommand{\clabel}{{y}}

\newcommand{\abs}[1]{\left\vert#1\right\vert}

\def \argmin {\mathop{\rm arg\,min}}

\makeatletter
\newcommand{\newreptheorem}[2]{\newtheorem*{rep@#1}{\rep@title} 
	\newenvironment{rep#1}[1]{\def\rep@title{#2 \ref*{##1}}\begin{rep@#1}}{\end{rep@#1}}
}
\makeatother

\newreptheorem{lemma}{Lemma}
\newreptheorem{theorem}{Theorem}
\newreptheorem{claim}{Claim}
\newreptheorem{proposition}{Proposition}
\newreptheorem{corollary}{Corollary}

%% file: 0_abstract.tex
% !TEX root =  main.tex
%%%%%%%%%%%%%%%%%%%%%%%%%%%%%%%%%%%%%%%%%%%%%%%%%%%%%%%%%%
%%%%%%%%%%%%%%%%%%%%%%%%%%%%%%%%%%%%%%%%%%%%%%%%%%%%%%%%%%
\begin{abstract}
\looseness-11
Machine teaching studies the interaction between a teacher and a student/learner where the teacher selects training examples for the learner to learn a specific task. The typical assumption is that the teacher has perfect knowledge of the task---this knowledge comprises knowing the desired learning target, having the exact task representation used by the learner, and knowing the parameters capturing the learning dynamics of the learner. Inspired by real-world applications of machine teaching in education, we consider the setting where teacher's knowledge is limited and noisy, and the key research question we study is the following: When does a teacher succeed or fail in effectively teaching a learner using its imperfect knowledge? We answer this question by showing connections to how imperfect knowledge affects the teacher's solution of the corresponding machine teaching problem when constructing optimal teaching sets. Our results have important implications for designing robust teaching algorithms for real-world applications.
\end{abstract}

%% file: 1_introduction.tex
\vspace{-2mm}
\section{Introduction}\label{sec:intro}
The field of machine teaching studies the interaction between a teacher and a student/learner where the teacher's objective is to select a short sequence of examples for the learner to learn a specific task \cite{goldman1995complexity,DBLP:journals/corr/ZhuSingla18}.  An important application is in education where the \emph{learner} is a human student, and the \emph{teacher} is a computerized intelligent tutoring system (ITS) that selects curriculum of learning material for the student \cite{zhu2015machine,rafferty2016faster,sen2018machine,hunziker2019teaching}.  Another concrete application is the data poisoning (training-time) adversarial attacks where the \emph{learner} is a machine learning (ML) system, and the \emph{teacher} is a hacking algorithm that poisons the training data to maliciously change the learner's output to a desired target \cite{mei2015using,zhu2018optimal}.  Regardless of the application and the teacher's intentions, machine teaching provides a formal model of quantifying the teaching effort and an algorithmic framework for deriving an optimized curriculum of  material to have maximum influence on the learner with minimal effort. Considering applications in educational settings, the problem of designing optimized curriculum is of utmost importance because it leads to more effective learning, increased engagement, and reduced drop-out of students \cite{archambault2009student}. 
%%The field of machine teaching holds the promise of enhafield lies in providing a clean formalism to quantify the provides a formal model to 

%\paragraph{Key issue in applying machine teaching algorithms.} 
\looseness -1 The key issue in applying machine teaching algorithms to real-world applications is that these algorithms (and the corresponding theoretical guarantees) often make unrealistic assumptions about the teacher's knowledge of the learner and the task. It is typically assumed that the teacher has perfect knowledge of the following: (i) the learner, e.g., a computational model of the learning dynamics, and parameters capturing initial knowledge and learning rate, (ii) task specification, e.g., a complete ground truth data and  representation of the task as used by the learner.  Assuming such a powerful teacher might be meaningful for deriving theoretical guarantees (e.g., computing information-theoretic lower bounds of teaching complexity \cite{goldman1995complexity,zilles2011models,DBLP:conf/nips/ChenSAPY18,mansouri2019preference}) or for understanding the vulnerability of an ML system (e.g., against a white-box poisoning attack \cite{DBLP:conf/aaai/Zhang0W18,DBLP:conf/ijcai/Ma0H19}). However, for applications in education where the student is a human learner, this assumption is clearly unrealistic: %In educational settings, 
learning dynamics and task specifications are usually obtained from domain expertise or inferred from historic student data (see \cite{singla2014near,piech2015deep,settles2016trainable,sen2018machine,hunziker2019teaching}), and this information is often incomplete and noisy. 

\subsection{Our approach and contributions}
Ironically, while the promise of machine teaching algorithms lies in providing near-optimal teaching curriculum with guaranteed performance, the fundamental assumptions required by these algorithms are clearly violated in practice. The main research question we study in this paper is the following: \textit{When does a teacher succeed or fail in effectively teaching a learner using its imperfect knowledge?}
% \begin{quoteenv}
% \noindent \textit{When does a teacher succeed or fail in effectively teaching a learner using its imperfect knowledge?}
% \end{quoteenv}

\looseness -1 To answer this question, we require a formal specification of the task, a learner model, and a concrete teaching algorithm. In our work, we study a classification task in the context of teaching human students the rules to identify animal species---an important skill required for  biodiversity monitoring related citizen-science projects~\cite{sullivan2009ebird,van2018inaturalist}. This is one of the few real-world applications for which machine teaching algorithms with guarantees have been applied to teaching human learners in educational settings (see \cite{singla2013actively,singla2014near,chen18explain,mac2018teaching}) and hence is well-suited for our work.
%We formally state the task specification, learner model, and optimal teacher in Section~\ref{sec:model}. %While we focus on a particular specification in this paper, it is important to note that our ideas and results can be extended to more complex tasks and models \yuxin{here is a bit handwavy. Any hints on what/how it could be extended will be helpful. Either consider removing it, or perhaps adding a footnote or a parenthesis explanation, how about: ``... can be extended to more complex tasks and models (e.g., robustness in teaching a regression/IRL task, see Section 1.2)''?}. 
We highlight some of the key contributions and results below:

\setlength{\leftmargini}{0.62em}
\begin{itemize}
	\item We formally study the problem of robust machine teaching. %quantifying the robustness of machine teaching algorithms when the teacher's knowledge is imperfect. 
	To quantify the effectiveness of teaching, we introduce two metrics that measure teacher's success in terms of the learner's eventual error, %at the end of teaching, 
	and the size of the teaching set as constructed by teacher with imperfect knowledge (\secref{sec:model}).
	\item %When studying robustness w.r.t. noise in learning dynamics, 
	We show that teaching is much more brittle w.r.t noise in learning rate and less so when considering noise in prior knowledge of the learner. This theoretical result %Perhaps somewhat surprisingly, very 
	aligns with a similar observation recently made in the context of a very different learning model   (Section~\ref{sec:noise-in-learning-parameters}).
	\item When studying robustness w.r.t. noise in task specification, we provide natural regularity conditions %smoothness criterion 
	on the data distributions and then use these conditions when specifying the guarantees. This allows us to take a less pessimistic view %of the problem 
	than those considered by contemporary works which study the worst-case setting (Section~\ref{sec:noise-in-representation}).
\end{itemize}
\setlength{\leftmargini}{2.5em}

%% file: 2_relatedwork.tex
\subsection{Related work on robust machine teaching}  \label{sec:related}
%\annotate{The question of robust machine teaching posed above has gained  a lot of recent attention in a quest to designing algorithms that are applicable to real-world settings.}
%\looseness-1
%A growing body of contemporary works has tackled the problem of robust machine teaching in different forms, however, with a very different focus compared to ours. For instance, 
\cite{DBLP:conf/icml/LiuDLLRS18,DBLP:conf/ijcai/MeloGL18,DBLP:conf/icml/Dasgupta0PZ19,DBLP:conf/ijcai/KamalarubanDCS19} have studied the problem of teaching a ``blackbox" learner where the teacher has very limited or no knowledge of the learner. The focus of these papers has been on designing an \emph{online} teaching algorithm that infers the learner model in an online fashion. These works often conclude that an \emph{offline} teaching algorithm that operates with limited knowledge can perform arbitrarily bad by considering a worst-case setting.  However, designing and deploying online algorithms is a lot more challenging in practice---the results in contemporary works have mostly been theoretical and might not be directly applicable in practice given the high sample complexity of online inference. The focus of our work is primarily on \emph{offline} teaching algorithms where knowledge about the task is usually obtained from domain expertise or inferred from historic student data. We aim at developing a fundamental %getting a better 
understanding of how the performance guarantees of a teaching algorithm degrade w.r.t. the noise in teacher's knowledge when considering natural data distributions.

In another line of contemporary work on teaching a reinforcement learning agent,  \cite{DBLP:conf/nips/HaugTS18,DBLP:conf/nips/TschiatschekGHD19} have considered the setting where teacher and learner have different worldview and preferences---the focus of these works is on designing a teaching algorithm to account for these mismatches, and do not directly tackle the question we study in this paper. There has also been some empirical work on understanding the robustness and effect of different model components as part of the popular Bayesian Knowledge Tracing (BKT) teaching model used in ITS \cite{klingler2015performance,DBLP:conf/edm/KhajahLM16}---we see this work as complementary to ours as we take a more formal approach towards understanding the robustness of theoretical guarantees provided by machine teaching algorithms.

%% file: 3_statement.tex
\section{Problem Formulation}\label{sec:model}
%In this section, we formally state our problem. %our problem setting. 
In this section, we first introduce the task representation, the learner's model, and the teacher's optimization problem. Then, we formulate the problem of teaching with imperfect knowledge, and discuss the notions of successful teaching.
 
\subsection{Teaching task and representation}\label{sec:model:task}
%In this paper, 
We consider the problem of teaching a binary classification task. Let $\Instances$ denote a ground set of instances (e.g., images) and the learner uses a feature representation of $\phi:\Instances \to \R^d$. % for the objects. 
Let  $\Hypotheses$ be a finite class of hypotheses considered by the learner where each hypothesis $\hypothesis \in \Hypotheses$ is a function $\hypothesis: \Instances \rightarrow \{-1, +1\}$. As a concrete setting, $\Hypotheses$ could be the set of hypotheses of the form $\hypothesis(\instance) = \textnormal{sign}(\langle \theta_h, \phi(x) \rangle)$ where $\theta_h \in \mathbb{R}^d$ is the weight vector associated with hypothesis $\hypothesis$.
%As a concrete setting, $\Hypotheses$ could be the set of linear hypotheses on the $\phi$-induced space with $\hypothesis(\instance) = \textnormal{sign}(\langle \theta_h, \phi(x) \rangle)$ where $\theta_h \in R^d$ is the weight vector associated with hypothesis $\hypothesis$.
% which are parameterized by $\phi$
%\footnote{We focus on binary classification setting as this is the most well-studied setting in algorithmic teaching literature and provides a neat formal model to answer the questions we study in this paper. Our framework and results can be extended to more complex teaching settings such as multi-class classification or more complex learner models such as gradient/online learners \todo{cite}, also see discussion in Section~\todo{3, 4}}~

% that internally uses the representation of $\instance$ given by $\phi(\instance)$. 
%Both teacher and learner know the task representation given by instances $\Instances$, feature map $\phi$, and the hypothesis class $\Hypotheses$. 
%and we assume $|\Instances|$ is finite. 

%\subsection{Ground truth labels}
Each instance $\instance \in \Instances$ is associated with a ground truth label given by the function $\clabel^*:\Instances \to \{-1, 1\}$ and we denote the ground truth label of an instance $\instance$ as $\clabel^*(\instance)$. The ground truth labels given by $\clabel^*$ are not known to the leaner. We use $\Examples$ to denote instances with their labels where a labeled example $\example \in \Examples$ is given by $\example = (\instance, \clabel^*(\instance))$. 
%The ground truth labels given by $\clabel^*$ are known to the teacher, but not to the leaner.

As typically studied in machine teaching literature, we consider a realizable setting where there exists a hypothesis $\hstar \in \Hypotheses$ such that $\forall \instance, \hstar(\instance) = \clabel^*(\instance)$.\footnote{\looseness-1 This assumption is w.l.o.g.: In a non-realizable setting, the teacher could consider $\hstar \in \Hypotheses$ as a hypothesis with minimal error in terms of disagreement of labels w.r.t. the labels given by $\clabel^*$ and the results presented in this paper can be extended to this general setting.\label{footnote:realizable}}~The teacher's goal can then be stated as that of teaching the hypothesis $\hstar$ to the learner by providing a minimal number of labeled examples to the learner.
Before formulating the teacher's optimization problem of selecting labeled examples, we state the learning dynamics of the learner below.

%The teacher's goal can then be stated as that of teaching the hypothesis $\hstar$ to the learner by providing minimal number of labeled examples to the learner. Below we state the learning dynamics of the learner.

\subsection{Learner model}\label{sec:model:learner}
We consider a probabilistic learner model that generalizes the well-studied \emph{version space} models in classical machine teaching literature (see \cite{goldman1995complexity}).
At a high-level, the learner works as follows: During the learning process, the learner maintains a score for each hypothesis given by $Q(\hypothesis)$ capturing learner's belief of how good the hypothesis $\hypothesis$ is. Given $Q(\hypothesis)$, the learner acts probabilistically by drawing a hypothesis with probability $\frac{Q(\hypothesis)}{\sum_{\hypothesis' \in \Hypotheses} Q(\hypothesis')}$. Next, we discuss how the scores are updated.
%from the distribution given by $Q(\hypothesis)$ with appropriate normalization. 

%We will assume that $Q_0(\hypothesis) > q_{min}$. 
Before teaching starts, the learner's prior knowledge about the task is captured by initial scores given by $Q_0(\hypothesis)$. For simplicity and as considered in \cite{singla2014near,chen18explain}, we will assume that $Q_0(\hypothesis)$ is a probability distribution over $\Hypotheses$. After receiving a set of labeled examples $S = \{(\instance_s, \clabel_s)\}_{s=1, 2, \ldots}$ from the teacher, we denote the learner's score as $Q(\hypothesis | S)$ which are updated as follows:
\begin{align}
    Q(\hypothesis | S) = Q_0(\hypothesis) \cdot \Pi_{s=1, 2, \ldots, |S|}  J(\clabel_s | \hypothesis, \instance_s, \eta) \label{eq:learner-model}
\end{align}
where $J$ is a likelihood function parameterized by $\eta \in (0, 1]$. In this paper, we consider the following likelihood function given by:
\begin{align}\label{eq:J-expo}
    J(\clabel_s | \hypothesis, \instance_s, \eta)  = 
    \begin{cases}
    1- \eta & \text{for } \hypothesis(\instance_s) \neq \clabel_s \\
    1 & \text{o.w.}
    \end{cases}
\end{align}
Here, the quantity $\eta$ captures a notion of learning rate. This model reduces to a randomized variant of the classical learner model \cite{goldman1995complexity}) for $\eta=1$. The main results and findings in the paper also generalize to more complex likelihood functions such as the logistic functions considered by \cite{singla2014near,mac2018teaching}. 

%A primary quantify of interest in this paper is the learner's expected error after receiving a set of examples $S$. The learner's expected error is given by the following:
%\begin{align}
%    \Error(S) &= \sum_{\hypothesis \in \Hypotheses}  \frac{Q(\hypothesis | S)}{\sum_{\hypothesis' \in \Hypotheses} Q(\hypothesis' | S)} \error(h) \label{eq:error}
%    %\textnormal{ where } & \error(\hypothesis) = \frac{\sum_{\example=(\instance, \clabel) \in \Examples} 1_{\hypothesis(\instance) \neq \clabel}}{|\Examples|} \notag
%\end{align}
%where $\error(\hypothesis) = \frac{\sum_{\example=(\instance, \clabel) \in \Examples} 1_{\hypothesis(\instance) \neq \clabel}}{|\Examples|}$.

An important quantity of interest in this paper is the learner's expected error after receiving a set of examples $S$. 
Let $\error(\hypothesis) = \frac{\sum_{\example=(\instance, \clabel) \in \Examples} 1_{\hypothesis(\instance) \neq \clabel}}{|\Examples|}$ be the expected error of $\hypothesis$. The learner's expected error is given by the following:
\begin{align}
    \Error(S) &= \sum_{\hypothesis \in \Hypotheses}  \frac{Q(\hypothesis | S)}{\sum_{\hypothesis' \in \Hypotheses} Q(\hypothesis' | S)} \error(h) \label{eq:error} %\\
    %\textnormal{ where } & \error(\hypothesis) = \frac{\sum_{\example=(\instance, \clabel) \in \Examples} 1_{\hypothesis(\instance) \neq \clabel}}{|\Examples|} \notag
\end{align}
%where $\error(\hypothesis) = \frac{\sum_{\example=(\instance, \clabel) \in \Examples} 1_{\hypothesis(\instance) \neq \clabel}}{|\Examples|}$.

\subsection{Teaching with perfect knowledge}\label{sec:model:teacher-perfect}
\looseness -1 
We first consider the optimal teaching problem when the teacher has perfect knowledge of the teaching task represented as $(Q_0, \eta, \Examples, \hstar, \phi, \Hypotheses)$.
%We pose the problem of selecting optimal teaching set when teacher has perfect knowledge of the teaching task. 
In particular, the teacher's knowledge comprises: (i) learning dynamics captured by learner's initial knowledge $Q_0$ and learning rate $\eta$, (ii) task specification captured by the target hypothesis $\hstar$, the ground set of labeled examples $\Examples$, the feature map $\phi$, and hypothesis class $\Hypotheses$.

%knowledge about target hypothesis $\hstar$ inferred from the ground set of labeled examples $\Examples$, and (iii) task representation given by  feature map $\phi$ and hypothesis class $\Hypotheses$.
%through the ground truth labels $\clabel^*(\instance) \ \forall \  \Instances$
%. We assume that teacher has full knowledge about the task, i.e., it knows $\Instances$, $\phi$, $\Hypotheses$, $\clabels$, $Q_0$, and $\eta$. 

Teacher's primary goal is to find a smallest set of labeled examples to teach so that learner's error is below a certain desirable threshold $\epsilon$. In order to construct the optimal teaching set, instead of directly optimizing for reduction in error, it is common in literature to construct surrogate objective functions which capture learner's progress towards learning $\hstar$ (also, see \cite{goldman1995complexity,singla2014near,chen18explain,mac2018teaching}.). 
%which allow one to implement tractable algorithms as detailed below 
%Given teacher's perfect knowledge of $(Q_0, \eta, \Examples, \hstar, \phi, \Hypotheses)$, 

Let us define a set function $F:2^{\Examples} \to \R_{\geq 0}$ as follows:
%\Instances, \clabel^*
\begin{align}\label{eq:F}
    F(S) = \sum_{\hypothesis \in \Hypotheses} \Big(Q_0(\hypothesis) - Q(\hypothesis|S)\Big) \cdot \error(h)
    %w_\hypothesis
%\vspace{-4mm}
\end{align}
%where  $w_h \in R_+$ are some non-negative weights assigned to hypothesis. 
\looseness-1
Here, the quantity $\big(Q_0(\hypothesis) - Q(\hypothesis|S)\big)$ captures the reduction in score for hypothesis $\hypothesis$ after learner receives examples set $S$. In particular, the surrogate objective function $F$ is a soft variant of set cover, and allows one to design greedy algorithms to find near-optimal teaching sets. %\cite{goldman1995complexity,singla2014near,chen18explain,mac2018teaching}.)

For a given $\epsilon$, one can find a corresponding (sufficient) stopping value  $C_\epsilon$ such that $F(S)  \geq C_\epsilon$ implies that $\Error(S) \leq \epsilon$. As used in the optimization frameworks of \cite{singla2014near,mac2018teaching}, we use  $C_\epsilon = \sum_{\hypothesis \in \Hypotheses} Q_0(\hypothesis) \cdot \error(h) - \epsilon \cdot Q_0(\hstar)$.  This leads to the following optimization problem:
\begin{align}\label{eq:F-opt}
    \min_{S \subseteq \Examples} &|S| \ \text{s.t.} \ F(S) \geq \sum_{\hypothesis \in \Hypotheses} Q_0(\hypothesis) \cdot \error(h) - \epsilon \cdot Q_0(\hstar)
\end{align}
where $F(S)$ is given in Eq.~\ref{eq:F}.
%This optimization problem is very generic and captures the formulations used in the relevant machine teaching literature. 
%Given a fixed desirable performance of the learner $\epsilon$, 
We use $\opt_\epsilon$ to denote the optimal teaching set as a solution to the problem \eqref{eq:F-opt}.

%Then, ,  by ensuring that $F(S)  \geq C_\epsilon$ that leads to desirable performance ensuring that . 
%We cast the teacher's optimization problem as follows:
%\begin{align}
%    \min_{S} |S| \ s.t. \ F(S) \geq C_\epsilon
%    \label{eq:teach-opt}
%\end{align}
%for some constant $C_\epsilon$ that leads to desirable performance ensuring that learner's error is less than $\epsilon$. 

%%Following the optimizaiton framework of \cite{singla2014near,mac2018teaching}, we will use $C_\epsilon = \sum_{\hypothesis \in \Hypotheses} Q_0(\hypothesis) \cdot \error(h) - \epsilon \cdot Q_0(\hstar)$, leading to the following optimization problem:
%\begin{align}\label{eq:F-opt}
%    F(S) = \sum_{\hypothesis \in \Hypotheses} \Big(Q_0(\hypothesis) - Q(\hypothesis|S)\Big) \cdot \error(h)
%    %w_\hypothesis
%\end{align}
%
%For instance, in the classical teaching model \cite{goldman1995complexity}, we have $\eta = 1$, and one sets $w_\hypothesis = 1$ and $C = |\Hypotheses| - 1$. Then, $F(S)$ captures the number of hypothesis eliminated by set of examples $S$. 
%
%In the teaching model of \cite{DBLP:journals/corr/SinglaBBKK14}, $w_\hypothesis$ is set to be error of the $\hypothesis$ and $C$ is carefully choosen to achieve desired error. It is important to note that the function $F$ is submodular and the optimization problem above is an instance of set-cover problem. 
%which in turns define the corresponding factor $C_\epsilon$

\subsection{Teaching with imperfect knowledge}\label{sec:model:teacher-imperfect}
%In real-world applications, a teacher (e.g., an intelligent tutoring system in educational applications or an attacker in data-poisoning applications) would never have a perfect knowledge of the task and learning dynamics. 
%\todo{Add a few lines that existing work assumes access to true parameters which is not possible.}
%of $(Q_0, \eta, \Examples, \hstar, \phi, \Hypotheses)$. In particular, we 
We now consider a teacher with imperfect knowledge and study the following different settings:
\setlength{\leftmargini}{0.62em}
\begin{itemize}
    \item having noise on learner's initial knowledge $Q_0$ ( Section~\ref{sec:noise-in-learning-parameters:prior})
     \item having noise on learner's learning rate $\eta$ (Section~\ref{sec:noise-in-learning-parameters:rate}).
     %. In particular, we consider a teacher who overestimates or underestimates the learning rate
    \item having access to ground truth labels for only a subset of instances instead of the whole ground set $\Instances$ (Section~\ref{sec:noise-in-representation:sampleX}).
    %. This in turn implies that the teacher has possibly a wrong estimate of $\hstar$, there is noise in the $\error(h)$ for hypotheses, and the set that teacher can use to pick teaching examples is limited
 % not knowing the true hypothesis $\hstar$. This could arise when teacher has access to ground truth labels for only a small subset of instances $\Instances_T \subset \Instances$, or when teacher is considering an imperfect hypothesis space, for example, $\Hypotheses_T \subset \Hypotheses$.    
    \item having a noisy feature map, i.e., teacher's assumed feature map does not match with $\phi$ used by the learner (Section~\ref{sec:noise-in-representation:feature}).
\end{itemize}
%the true feature map

\looseness-1 We denote teacher's view of the imperfect knowledge as $(\widetilde{Q}_0, \widetilde{\eta}, \widetilde{\Examples}, \widetilde{\hstar}, \widetilde{\phi}, \widetilde{\Hypotheses})$. Given this knowledge, the teacher has its own view of quantities such as $\widetilde{Q}(\hypothesis | S)$ (cf., Eq.~\ref{eq:learner-model}),  $\widetilde{\error}(\hypothesis)$ (cf., $\error(.)$ used in Eq.~\ref{eq:error}), and $\widetilde{F}$ (cf., Eq.~\ref{eq:F}) as counterparts to those of a teacher with perfect knowledge. The optimization problem from the viewpoint of the teacher with imperfect knowledge can be written as follows: %\yuxin{\sout{(cf. Eq.~\ref{eq:F-opt}):}}
\begin{align}\label{eq:F-opt-imperfect}
    \min_{S \subseteq \widetilde{\Examples}} &|S| \ \text{s.t.} \ \widetilde{F}(S) \geq \sum_{\hypothesis \in \widetilde{\Hypotheses}} \widetilde{Q}_0(\hypothesis) \cdot \widetilde{\error}(h) - \epsilon \cdot \widetilde{Q_0}(\widetilde{\hstar})
\end{align}

In the subsequent sections, we will introduce notions of $\Delta$-imperfect knowledge depending on a set/tuple of parameters $\Delta$.
Let us denote by  $\apx_{\epsilon,\Delta}$ the teaching set found by $\Delta$-imperfect teacher as
a solution to the problem \eqref{eq:F-opt-imperfect}. The following definitions quantify the success of a teacher with imperfect knowledge w.r.t. to measure $M.1$ (related to learner's error) and measure $M.2$ (related to teaching set size).

\begin{definition}[$M.1$-successful]
\label{def:m1}
We say a teacher is $M.1$-successful if the learner's eventual error upon receiving the set $\apx_{\epsilon, \Delta}$ is $O(\epsilon)$ (here we treat the parameters as constant).
\end{definition}

\begin{definition}[$M.2$-successful]
\label{def:m2}
We say a teacher is $M.2$-successful if $|\apx_{\epsilon, \Delta}| \leq |\opt_{\hat{\epsilon}}|$, where $\hat{\epsilon}=\Theta(\epsilon)$ (here we treat the parameters as constant). In other words, the size of the teacher's teaching set
%we are saying that the teacher's set size 
is competitive w.r.t. that of a teacher with perfect knowledge which constructs an optimal teaching set for an $\Theta(\epsilon)$ error threshold.\footnote{This is the style of bound often considered in literature when taking an optimization perspective on teaching \cite{singla2014near,chen18explain}. One might be tempted to directly bound the size $|\apx_{\epsilon, \Delta}|$ as a function of $|\opt_{\epsilon}|$, however this is usually not possible without making further assumptions about the data distribution.}
\end{definition}

%% file: 4_noise-in-learning-parameters.tex
\section{Imperfect Knowledge about the Dynamics}
%Learning Parameters
\label{sec:noise-in-learning-parameters}
In this section, we explore the effectiveness of teaching when the teacher has imperfect knowledge of learner's initial knowledge $Q_0$ and learner's learning rate $\eta$. In fact, these two parameters are key to many popular learner models (e.g., the popular Bayesian Knowledge Tracing (BKT) models in educational applications \cite{piech2015deep,klingler2015performance,DBLP:conf/edm/KhajahLM16}, spaced-repetition models used in vocabulary applications \cite{settles2016trainable,hunziker2019teaching}, or gradient learner models studied for data-poisoning attacks \cite{DBLP:conf/icml/LiuDLLRS18}). %In real-world educational applications, these parameters are often inferred through historic data, and for theoretical studies it is often assumed that true parameters are known.
%%%%%%%%%%%%%%%%%%%%%%%%%%%%%%%%%%%%%%%
%%%%%%%%%%%%%%%%%%%%%%%%%%%%%%%%%%%%%%%

% The analysis that we provide below sheds answer to question of robustness of the teaching when these parameters are known only up to certain accuracy.
%\annotate{shrink here?}
\subsection{Noise in learning prior}
\label{sec:noise-in-learning-parameters:prior}
Here, we consider the setting where the teacher has a noisy estimate $\widetilde{Q_0}$ of learner's initial distribution $Q_0$, i.e.,  $(\widetilde{Q}_0, \widetilde{\eta}, \widetilde{\Examples}, \widetilde{\hstar}, \widetilde{\phi}, \widetilde{\Hypotheses}) := (\widetilde{Q_0}, \eta, \Examples, \hstar, \phi, \Hypotheses)$. The following definition quantifies the noise in $\widetilde{Q_0}$ w.r.t. the true $Q_0$.
%the The teacher's knowledge here is represented as $(\widetilde{Q_0}, \eta, \Examples, \hstar, \phi, \Hypotheses)$. Below, we formally quantify 

\begin{definition}[$\Delta_{Q_0}$-imperfect]\label{def:noise-in-learning-parameters:prior}
Let $\Delta_{Q_0} = (\delta_1, \delta_2)$ for $\delta_1,\delta_2 \geq 0$. We say that teacher's estimated distribution $\widetilde{Q}_0$ is $\Delta_{Q_0}$-imperfect if the following holds: 
\begin{align*}
\forall \ \hypothesis \in \Hypotheses, (1 - \delta_1) \cdot Q_0(\hypothesis) \leq \widetilde{Q_0}(\hypothesis) \leq Q_0(\hypothesis) \cdot (1 + \delta_2).
\end{align*}
%given that $\widetilde{Q}_0$ is a valid distribution over the hypotheses classes $\Hypotheses$
%where $\widetilde{Q_0}(.)$ and $Q_0(.)$ are probability distributions over the hypotheses classes $\Hypotheses$.
\end{definition}
% Note that learner's initial knowledge $Q_0$ is a probability distribution over the hypothesis class $\Hypotheses$ and  we assume that teacher's estimated parameter $\widetilde{Q_0}(.)$ is a valid distribution as well. 
 
 The following theorem quantifies the effectiveness of teaching w.r.t. measures $M.1$ and $M.2$ (see Definitions~\ref{def:m1}, ~\ref{def:m2}).% in terms of the two success criteria C.1 and C.2 discussed in the previous section.
%in this setting 

\begin{theorem}
\label{thm:noise-in-learning-parameters:prior}
Fix $\epsilon \geq 0$, $\delta_1  \geq 0$, and $\delta_2 \geq 0$. Consider a teacher with knowledge $(\widetilde{Q}_0, \eta, \Examples, \hstar, \phi, \Hypotheses)$, where $\widetilde{Q}_0$ is $\Delta_{Q_0}$-imperfect w.r.t. true $Q_0$ for $\Delta_{Q_0}=(\delta_1, \delta_2)$. Then, in the worst-case for any problem setting and any $\Delta_{Q_0}$-imperfect $\widetilde{Q}_0$, the teacher is successful w.r.t. measures $M.1$ and $M.2$ with the following bounds:
\begin{enumerate}
    \vspace{-1mm}
    \item The learner's error is $O(\epsilon)$ and is bounded as $\Error(\widetilde{\opt}_{\epsilon, \Delta_{Q_0}}) \leq \frac{\epsilon \cdot (1 + \delta_2)}{(1 - \delta_1)}$.
    \vspace{-1mm}    
    \item The size of the teaching set is bounded as $|\widetilde{\opt}_{\epsilon, \Delta_{Q_0}}| \leq |\opt_{\hat{\epsilon} }| \textnormal{ where } \hat{\epsilon} = \frac{\epsilon \cdot (1 - \delta_1)}{(1 + \delta_2)}.$
%\begin{align*}
%|\widetilde{\opt}_{\epsilon, \Delta_{Q_0}}| \leq |\opt_{\hat{\epsilon} }| \textnormal{ where } \hat{\epsilon} = \frac{\epsilon \cdot (1 - \delta_1)}{(1 + \delta_2)}.
%\end{align*}
%$$|\overline{\opt}_{\epsilon, \delta_1, \delta_2}| \leq |\opt_{\frac{\epsilon \cdot (1 - \delta_1)}{(1 + \delta_2)}}|.$$
\end{enumerate}
\end{theorem}
\vspace{-1mm}  
% that satisfies Definition~\ref{def:noise-in-learning-parameters:prior}
  
%\begin{theorem}
%\label{thm:noise-in-learning-parameters:prior}
%Consider a fixed $\epsilon \geq 0$, $\delta_1  \geq 0$, and $\delta_2 \geq 0$. Consider a teacher with imperfect knowledge of learner's prior knowledge and all other parameters perfectly known, i.e., representing teacher's knowledge as  $(\widetilde{Q}_0, \eta, \Examples, \hstar, \phi, \Hypotheses)$ where $\widetilde{Q}_0$ is $\Delta=\{\delta_1, \delta_2\}$ noisy w.r.t. true $Q_0$. Then, in the worst-case for any  $\Delta$ noisy estimate $\widetilde{Q}_0$, the teacher is successful w.r.t. criteria C.1 and C.2 with the following bounds:
%\begin{enumerate}
%    \item Lhe learner's error is upper bounded by $O(\epsilon)$ given by $$\Error(\widetilde{\opt}_{\epsilon, \Delta=\{\delta_1, \delta_2\}}) \leq \frac{\epsilon \cdot (1 + \delta_2)}{(1 - \delta_1)}.$$
%    \item The size of the teaching set is bounded as follows:
%$$|\overline{\opt}_{\epsilon, \delta_1, \delta_2}| \leq |\opt_{\hat{\epsilon} }|.$$
%where $\hat{\epsilon} = \frac{\epsilon \cdot (1 - \delta_1)}{(1 + \delta_2)}$.
%%$$|\overline{\opt}_{\epsilon, \delta_1, \delta_2}| \leq |\opt_{\frac{\epsilon \cdot (1 - \delta_1)}{(1 + \delta_2)}}|.$$
%\end{enumerate}
%\end{theorem}
The proofs are provided in the Appendix.
%\footnote{Dropbox: \url{https://tinyurl.com/ijcai2020-robust-teaching-sup}}
%The proofs are provided in the longer version of the paper.\footnote{Dropbox: \url{https://tinyurl.com/ijcai2020-robust-teaching-sup}} %The proof (provided in longer version, see Footnote~\ref{footnote:longer}) 
%The proof builds up on the existing theory of machine teaching  ( \cite{singla2014near,chen18explain}) and essentially bounds the sensitivity of function $\widetilde{F}$ w.r.t. function $F$ in terms of noise in $Q_0$.

%{\bfseries Remark. } The Theorem~\ref{thm:noise-in-learning-parameters:prior}'s bound on the size of the teaching set is to be interpreted as follows: it is competitive w.r.t. an optimal noise-free teacher which constructs teaching set for  $O(\epsilon)$ error threshold -- this is the style of bound often considered in literature when taking optimization perspective on teaching \todo{cite}. A more directly interpretable result one would like to get is to directly bound the size in terms of $|\opt_{\hat{\epsilon} }|$, however such a generic bound is usually not possible without making further assumptions about data distribution.
%\todo{Perhaps move the success creterias to previous section¡}

\subsection{Noise in learning rate}
\label{sec:noise-in-learning-parameters:rate}
Next, we consider a setting where the teacher has an  imperfect estimate of the learner's learning rate $\eta$ while having perfect knowledge about the rest of the parameters, i.e., the teacher's knowledge is $(Q_0, \widetilde{\eta}, \Examples, \hstar, \phi, \Hypotheses)$. The following definition quantifies the noise in $\widetilde{\eta}$ w.r.t. true $\eta$.
%Below, we formally quantify the noise in $\widetilde{\eta}$ w.r.t. true $\eta$.

\begin{definition}[$\Delta_{\eta}$-imperfect]\label{def:noise-in-learning-parameters:rate}
Let $\Delta_\eta = (\delta)$ for $\delta \geq 0$. We say that a teacher's estimate $\widetilde{\eta}$ is $\Delta_\eta$-imperfect if $\abs{\widetilde{\eta} - \eta} \leq \delta$, where both $\widetilde{\eta} \in (0, 1]$ and $\eta \in (0, 1]$.
\end{definition}
%In particular, for a fixed  $\Delta_\eta = (\delta)$, 
The following two worst-case scenarios are of interest: (i) a teacher who overestimates the learning rate with $\widetilde{\eta} = \min\{\eta + \delta, 1\}$ and (ii) a teacher who underestimates the learning rate with $\widetilde{\eta} = \max\{\eta - \delta, 0\}$. The following theorem quantifies the challenges in teaching successfully in this setting.
% in terms of the two measures $M.1$ and $M.2$.
  
\begin{theorem}
\label{thm:noise-in-learning-parameters:learningrate}
Fix $\epsilon \geq 0$ and $\delta \geq 0$. Consider a teacher with knowledge  $(Q_0, \widetilde{\eta}, \Examples, \hstar, \phi, \Hypotheses)$ where $\widetilde{\eta}$ is $\Delta_{\eta}$-imperfect w.r.t. true $\eta$ for $\Delta_{\eta}=(\delta)$. Then, for any  $\Delta_{\eta}$-imperfect $\widetilde{\eta}$, there exists a problem setting such that the teacher is unsuccessful w.r.t. measures $M.1$ and $M.2$:
\begin{enumerate}
    \item  For any fixed  $\epsilon$ and $\Delta_\eta$, there exist problem settings where $\Error(\widetilde{\opt}_{\epsilon, \Delta_\eta}) \geq \frac{1}{2}$.
    \item For any fixed  $\epsilon$ and $\Delta_\eta$, and any $\hat{\epsilon}$ arbitrarily close to $0$, there exist problem settings where $|\widetilde{\opt}_{\epsilon, \Delta_\eta}| \geq |\opt_{\hat{\epsilon}}|$.
\end{enumerate}
\end{theorem}

Comparing Theorem~\ref{thm:noise-in-learning-parameters:prior} and Theorem~\ref{thm:noise-in-learning-parameters:learningrate}, these results suggest that noise in the teacher's assumption about the learning rate is a lot more hazardous compared to noise about the learner's initial distribution. While we derived these results by focusing on a very specific task and learner model, similar observations were made in the context of a different type of teaching setting when teaching a gradient learner \cite{DBLP:conf/aaai/YeoKSMAFDC19}.
%learning model and specific task of binary classification
%compared to noise in the assumption about the learner's initial distribution

Theorem~\ref{thm:noise-in-learning-parameters:learningrate} only provides a pessimistic view that teacher can fail badly. On closer inspection, the negative results arise from two separate issues: (i) teacher computing wrong utility of examples in \eqref{eq:F-opt-imperfect}, and (ii) teacher having a wrong estimate of stopping criteria in \eqref{eq:F-opt-imperfect} which in turn depends on learner's progress. Empirically, we found that the second reason seems to be the dominant one for the teacher's failure. One practical way to fix this issue is to develop an interactive teaching strategy where the teacher's stopping criteria is determined by the learner's true  progress measured in an online fashion instead of the progress as estimated by the teacher using its offline model (also, see discussions in Section~\ref{sec:related}).

%% file: 5_noise-in-representation.tex
\begin{figure*}[t!]
\centering
	\begin{subfigure}[b]{0.31\textwidth}
	    \centering
		\includegraphics[width=0.7\linewidth]{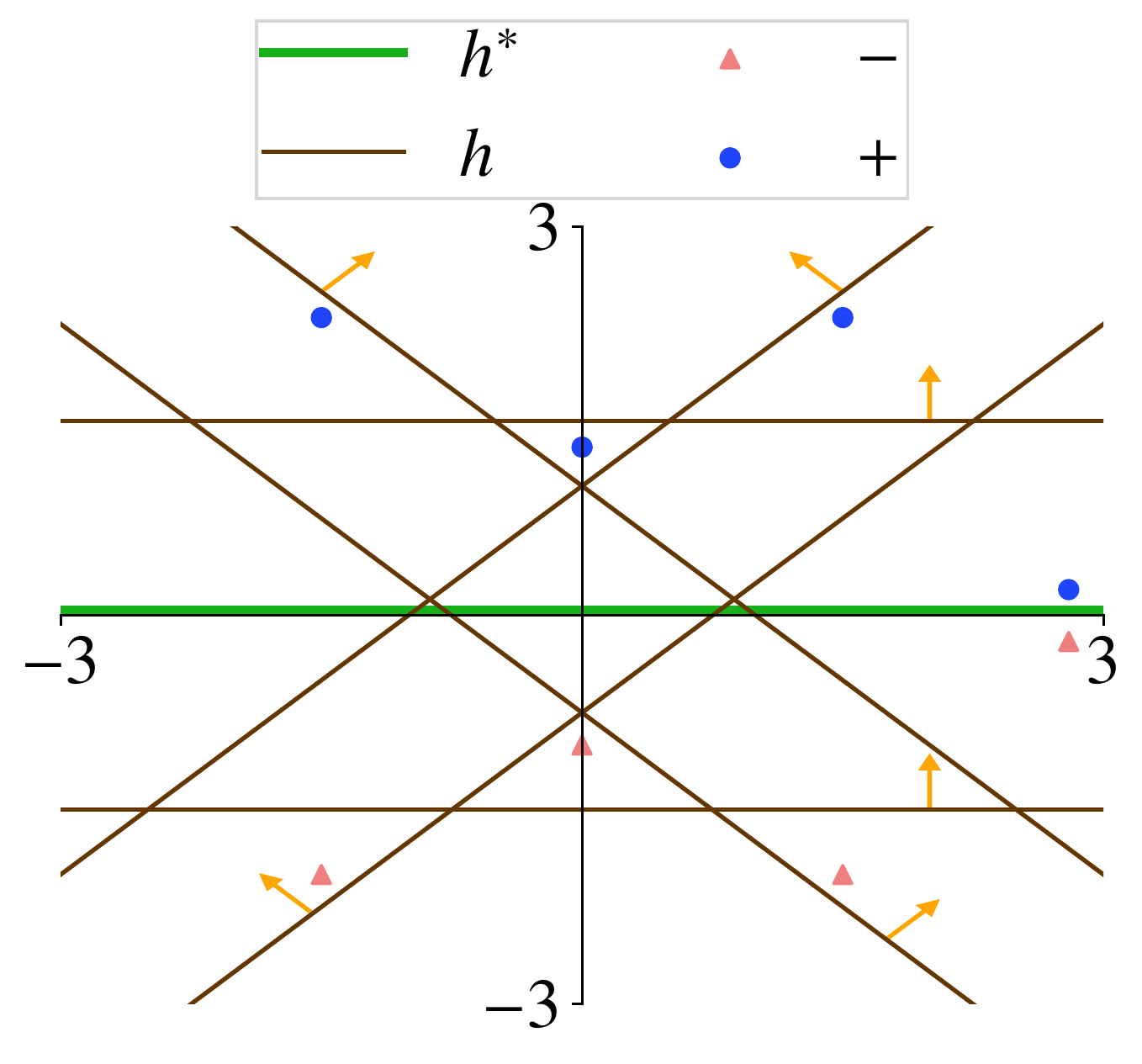}
		\caption{Problem with a few extreme data points}
		\label{fig:environments.bad1}
	\end{subfigure}		
	\quad
% 	\begin{subfigure}[b]{0.18\textwidth}
% 	    \centering
% 		\includegraphics[height=3cm]{fig/problem-instances/ipad-problem4.pdf}
% 		\caption{Th5 assum-2}
% 		\label{fig:environments.4}
% 	\end{subfigure}
	\begin{subfigure}[b]{0.31\textwidth}
	    \centering
		\includegraphics[width=0.7\linewidth]{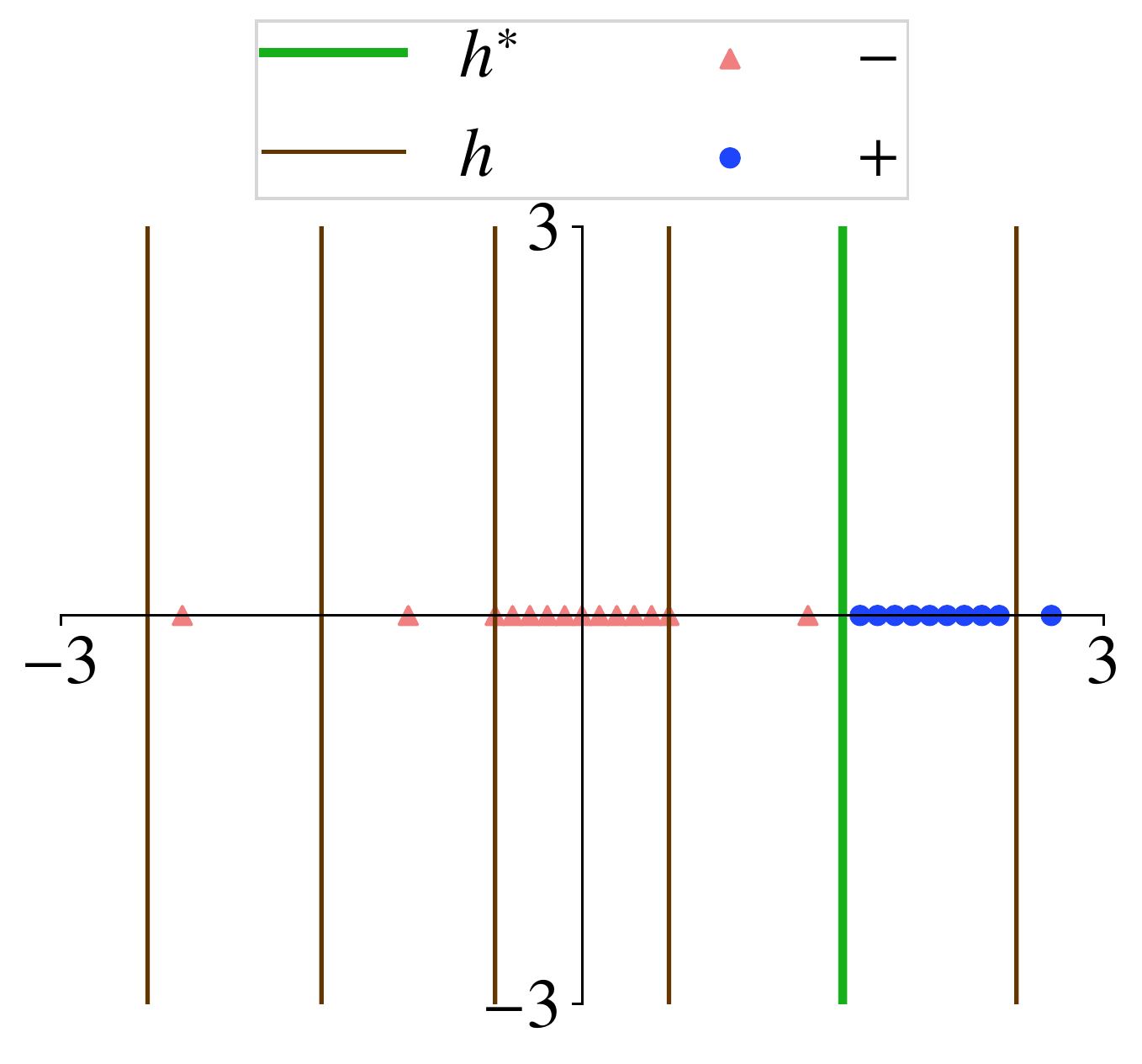}		
		\caption{Problem with skewed data distribution}
		\label{fig:environments.bad2}
	\end{subfigure}	
	\quad
	\begin{subfigure}[b]{0.31\textwidth}
	    \centering
		\includegraphics[width=0.7\linewidth]{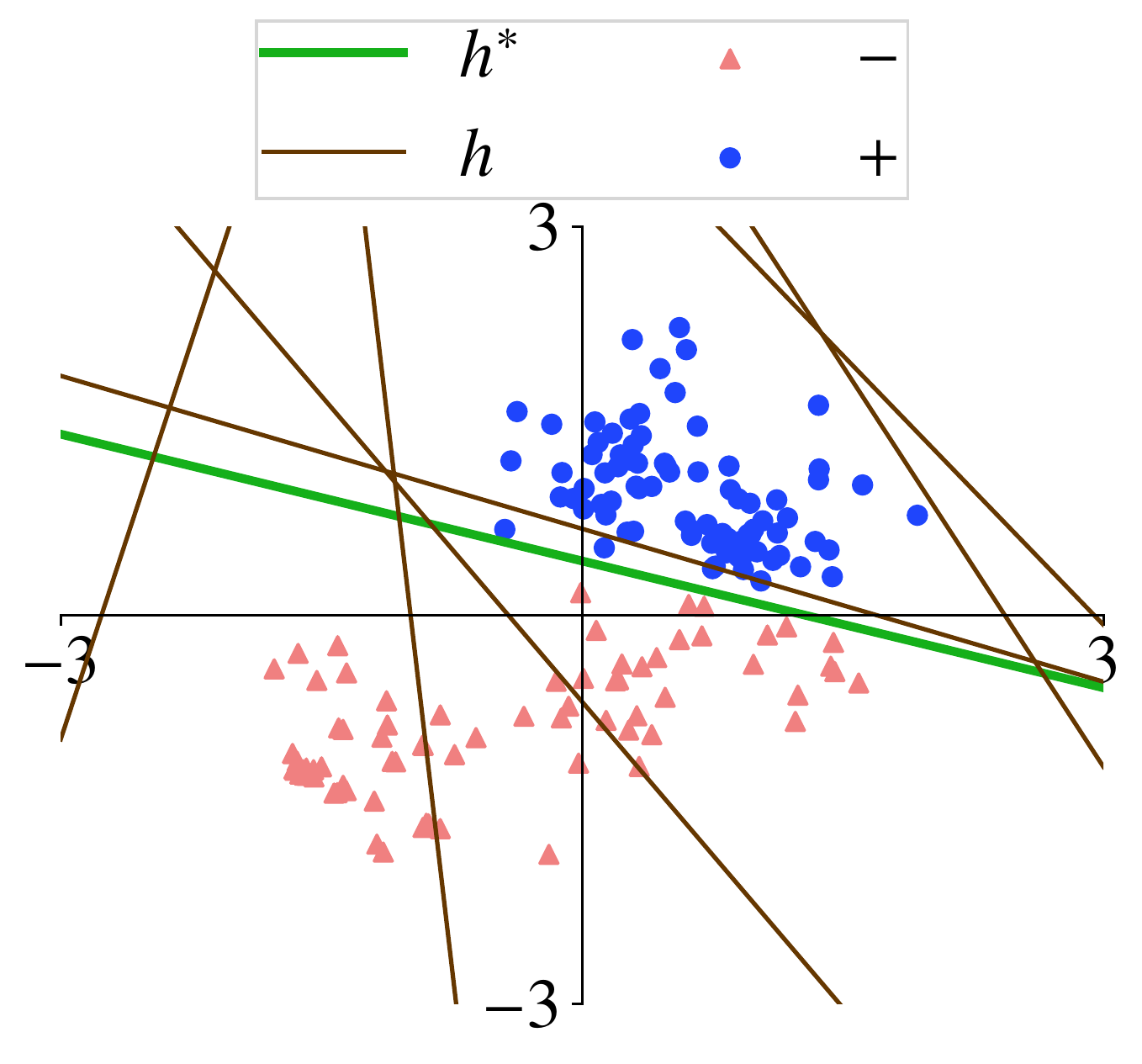}		
		\caption{Real-world problem}
		\label{fig:environments.butterflymoth}
	\end{subfigure}		
%	\quad
%	\begin{subfigure}[b]{0.22\textwidth}
%	   \centering
%		\includegraphics[height=3cm]{fig/problem-instances/ipad-problem1.pdf}
%		\caption{1-D threshold}
%		\label{fig:environments.1}
%	\end{subfigure}
	\vspace{-3mm}
   \caption{\textbf{(a)} shows a  problem setting with a few extreme points that are important for teaching---assuming $\eta=1$ and $\epsilon=0$ for simplicity of arguments, the optimal teaching set consists of only two examples lying close to coordinates $(+3,0)$. However, if these two examples are not present in $\widetilde{\Examples}$, the teaching set can be arbitrarily large (i.e., of size 6 in the illustration). \textbf{(b)} shows another problem setting with skewed data distribution where a small perturbation of data could lead to big changes in prediction of hypotheses. \textbf{(c)} shows a real-world problem setting to distinguish animal species. As can be seen, the data distribution here is more ``well-behaved" and does not suffer from issues present in the other two problem settings. Note that only a few hypotheses are shown in the illustration, see Section~\ref{sec:experiments} for more details).
  }
\label{fig:environments}
%\vspace{-5mm}
\end{figure*}
%(see \cite{singla2014near,mac2018teaching})
%and  indeed a representative of what one would expect in real-world applications is represents a more real-world problem setting and infact is taken from \todo{cite} and we use in our experiments. Figure 1c shows an example of skewed data points when learning 1-D threshold and Figure 1d shows a more well-behaed problem setting of 1-D setting often used as a canocical example.
%lies in clusters and hypotheses are well-separating these clusters
%%%%%%%%%%%%%%%%%%%%%%%%%%%%%%%%%%%%%

%%%%%%%%%%%%%%%%%%%%%%%%%%%%%%%%%%%%%%%%%%%%%%%%%%%%%%%%%%
\section{Imperfect Knowledge about Representation}
\label{sec:noise-in-representation}
In this section,  we explore the effect of teaching when the teacher has imperfect knowledge of the task specification, in particular, limited ground truth data and noisy representation of the task  used by the learner.  
%In real-world applications, this is usually the case where teacher acquires ground truth knowledge  from expert labels which are costly to obtain and infer representations from from limited historic student data.

%we explore the effective of teaching when the teacher has imperfect knowledge of the task. First we are exploring a setting where teacher has access to only a subset of ground truth labels for examples $\widetilde{\Examples}$ instead of assuming that teacher knows the true labeling function $\clabel^*(.)$ (\todo{cite}). In real-world applications, this is usually the case where teacher acquires ground truth knowledge from some expert labels which are costly and limited. Second, we explore a setting where the teacher only has a noisy representration of learner's feature map $\phi(.)$ and hypotheses class $\Hypotheses$. Again, in real-world applications of machine teaching algorithms , it is typically the case that such representations are learnt from historic data. Beyond the specific learner model we study in this paper, these two notions of imperfect knowledge about task representation are common in other models (e.g., Bayesian tracing models for education \todo{cite deep-bkt} and graditne learners \todo{cite}). 
%\annotate{add regularity condition?}

%%%%%%%%%%%%%%%%%%%%%%%%%%%%%%%%%%%%%%%%%%%%%%%%%%%%%%%%%%
\subsection{Limited ground truth labels}\label{sec:noise-in-representation:sampleX}
Here, we consider the setting where teacher has ground truth labels for only a subset of examples $\widetilde{\Examples} \subseteq \Examples$. The typical process followed when applying machine teaching algorithms is to first sample a small set of instances $\widetilde{\Instances} \subseteq \Instances$ and then get expert annotations to obtain $\widetilde{\Examples} \subseteq \Examples$ (e.g., see \cite{singla2014near,mac2018teaching}. Then, the teacher selects a hypothesis $\widetilde{\hstar}$ as the one with minimal empirical error given by $\widetilde{\hstar} \in \argmin_{\hypothesis \in \Hypotheses} \widetilde{\error}(\hypothesis)$. For this setting, we represent the knowledge of the teacher as $(Q_0, \eta, \widetilde{\Examples}, \widetilde{\hstar}, \phi, \Hypotheses)$. 
% where the function $\widetilde{\error(.)}$ is given in Eq.~\ref{eq:F-imperfect}

\looseness -1 As long as the set  $\widetilde{\Examples}$ is constructed \emph{i.i.d.}, the teacher can construct teaching sets to ensure that the learner's error would be low (i.e., teaching is successful w.r.t. measure $M.1$).  This argument follows from the standard concentration inequalities which ensures that with high probability, the teacher has a good estimate of $\widetilde{\error}(\cdot)$, i.e., $\forall \ \hypothesis \in \Hypotheses, |\widetilde{\error}(h) - \textnormal{err}(h)|$ is small (see  Theorem~\ref{thm:noise-in-representation:sampleX2}). However, regarding teacher's performance on measure $M.2$, without any additional assumptions about data distribution, it is easy to construct a pessimistic scenario %worst-case problem setting 
where the data distribution is skewed and the teaching set $\widetilde{\opt}_{\epsilon, \Delta}$ constructed by a teacher with imperfect knowledge is arbitrarily large w.r.t. the optimal teaching set  $\opt_{\epsilon}$, see Figure~\ref{fig:environments}.
%  We further elaborate on this point through some illustrative problem settings as shown in Figure~\ref{fig:environments} (see figure caption for discussion).% as discussed below.
% where the first bound on learner's error follows from these arguments without requiring any additional assumptions

Building on insights from the problem settings discussed in Figure~\ref{fig:environments}, we consider additional structural assumptions on the problem setting as discussed below. First, we introduce the notion of $\delta$-perturbed set of examples.
%which avoids issues as seen in Figure~\ref{fig:environments.bad1} and Figure~\ref{fig:environments.bad2}, and reflects properties often seen in real-world data distributions. 

\begin{definition}[$\delta$-perturbed]
\label{def:delta_perturbation}
Consider a set of labeled examples $S \subseteq \Examples$.  We call $S'$ a $\delta$-perturbed version of $S$, if there exists a bijective map $S \mapsto S', (\instance, \clabel) \mapsto (\instance', \clabel)$ such that $\lVert\phi(\instance) - \phi(\instance')\rVert_2 \leq \delta$.
\end{definition} 

%To formally state these properties, we first introduce the notion of $\delta$-perturbed set of examples. For any set of labeled examples $S \subseteq \Examples$, we call $S'$ a $\delta$-perturbed version of $S$, if each example $(\instance, \clabel) \in S$ maps uniquely to a perturbed $(\instance', \clabel) \in S'$ such that \begin{align}
%    \lVert\phi(\instance) - \phi(\instance')\rVert_2 \leq \delta.\label{eq:delta_perturbation}
%\end{align}

%%%%%%%%%%%%%%%%%%%%%%%%%%%%%%%%%%%%%%%%
%Definition~\ref{def:noise-in-representation:smoothness} below introduces a smoothness assumption about data distribution which is additionally needed for the results in Theorem~\ref{thm:noise-in-representation:sampleX2} (for bounding the size) and Theorem~\ref{thm:noise-in-representation:feature} (for bounding both the error and size). 

%Additionally, we consider the following smoothenss assumption on data distribution which would be useful for deriving meaningful bounds in 

We will also need the following smoothness notion for proving robustness guarantees (for bounding the size in Theorem~\ref{thm:noise-in-representation:sampleX2} and for bounding both the error/size in Theorem~\ref{thm:noise-in-representation:feature}). 
% \begin{definition}[$\lambda$-smoothness] \label{def:noise-in-representation:smoothness}
% Let $\delta \geq 0$, $\lambda \geq 0$. We call a problem setting $\lambda$-smooth, if for any $S \subseteq \Examples$ and every $ $
% Consider any set $S \subseteq \Examples$, and let $S'$ be a $\delta$-perturbed version of $S$. Then, we call the problem setting $\lambda$-smooth if for any $\hypothesis \in \Hypotheses$, the mismatch in labels assigned by  $\hypothesis$ to examples $S$ and $S'$ is upper-bounded by $\lambda\cdot\delta$.
% \end{definition}
\begin{definition}[$\lambda$-smoothness] \label{def:noise-in-representation:smoothness}
Let $\delta \geq 0$, $\lambda \geq 0$. Consider any set $S \subseteq \Examples$, and let $S'$ be any $\delta$-perturbed version of $S$. Then, we call the problem setting $\lambda$-smooth if for any $\hypothesis \in \Hypotheses$, the mismatch in labels assigned by  $\hypothesis$ to examples $S$ and $S'$ is upper-bounded by $\lambda\cdot\delta$.
\end{definition}

%define a $\delta$-perturbed version of $S$ denoted as $S'$. 
%Then,  our smoothness condition parametrized by a coefficient $\lambda$ says that for any $\hypothesis \in \Hypotheses$, the mismatch in labels assigned by 
%$\hypothesis$ to examples $S$ and $S'$ is upper-bounded by $\lambda.\delta$.
%Then, $\forall \ \hypothesis ||\hypothesis(S) - \hypothesis(S')||_1 \leq \lambda.\delta$ for some smoothness constant $\lambda$ where $||\hypothesis(S) - \hypothesis(S')||_1$ measures the number of label flips between for any $\hypothesis$ between $S$ and $S'$.
%where each example $(\instance, \clabel) \in S$ maps to a perturbed $(\instance', \clabel) \in S'$ such that $||\phi(\instance) - \phi(\instance')||_2 \leq \delta$
%Such are used in previous works (\todo{cite}). 
%We introduce these assumptions in the form of 

% as defined in \eqref{eq:delta_perturbation}

Definition~\ref{def:noise-in-representation:sampleX} below quantifies the imperfection in teacher's knowledge arising from the sampling process coupled with additional structural conditions. 
%---for this to be satisfied, the underlying data distribution implicitly requires to be have certain regularity assumptions (see Figure~\ref{fig:environments}).

% and we restrict attention to setting where $|\widetilde{\Examples}| \geq |\opt_\epsilon|$
\begin{definition}[$\Delta_{\Examples}$-imperfect]\label{def:noise-in-representation:sampleX}
Let $\Delta_{\Examples} = (\delta_1, \delta_2, \delta_3)$ for $\delta_1, \delta_2, \delta_3 \geq 0$. We say that a teacher's knowledge is $\Delta_{\Examples}$-imperfect if the following statements hold with probability at least $(1 - \delta_1)$:
\setlength{\leftmargini}{0.62em}
\begin{itemize}
\item $\forall \ \hypothesis, |\widetilde{\error}(h) - \textnormal{err}(h)| \leq \delta_2$,
\item for any set of labeled examples $S \subseteq \Examples$ with $|S| \leq |\widetilde{\Examples}|$, there exists a $\delta_3$-perturbed version  of  $S$ in $\widetilde{\Examples}$. 
\end{itemize}
\setlength{\leftmargini}{2.5em}
%\begin{gather}
%\forall \ \hypothesis, |\widetilde{\error}(h) - \textnormal{err}(h)| \leq \delta_2,\notag \\
%\forall \ S \subseteq \Examples, |S| \leq |\widetilde{\Examples}|, \exists S' \subseteq \widetilde{\Examples} \textnormal{ s.t. } S' \textnormal{ is } \delta_3\textnormal{-perturbed version}  of  S. \notag
%\end{gather}
%where a $\delta$-perturbed version of $S$ means that each example $(\instance, \clabel) \in S$ maps uniquely to a perturbed $(\instance', \clabel) \in S'$ such that $||\phi(\instance) - \phi(\instance')||_2 \leq \delta$
\end{definition}
%%\forall \ \example = (\instance, \clabel) \in \Examples, \exists \example' = (\instance', \clabel) \textnormal{ s.t. } ||\phi(\instance) - \phi(\instance')|| \leq \delta_3 \notag

%Note that  stated in the definition follows naturally from \emph{i.i.d.} sampling process.
Note that in the above definition, the bound on error is satisfied from the \emph{i.i.d.} sampling process and doesn't require any further structural assumption. The second condition implicitly adds regularity conditions on the underlying data distribution which should not have characteristics as seen in Figure~\ref{fig:environments.bad1} and Figure~\ref{fig:environments.bad2}. 
The following theorem quantifies the effectiveness of a $\Delta_{\Examples}$-imperfect teacher.

\begin{theorem}
\label{thm:noise-in-representation:sampleX2}
Fix $\epsilon \geq 0$ and $\Delta_{\Examples} = (\delta_1, \delta_2, \delta_3)$ with $\delta_1, \delta_2, \delta_3  \geq 0$. Consider a $\Delta_{\Examples}$-imperfect teacher with knowledge $(Q_0, \eta, \widetilde{\Examples}, \widetilde{\hstar}, \phi, \Hypotheses)$. Assume the problem setting is $\lambda$-smooth for some $\lambda \geq 0$, $\eta < 1$, and $|\widetilde{\Examples}|$ is sufficiently large. %Furthermore, assume that the problem setting to be $\lambda$-smoooth as per Definition~\ref{def:noise-in-representation:smoothness} and $\eta < 1$. 
Then, %in the worst-case 
for any sample $\widetilde{\Examples}$ and selection of $\widetilde{\hstar}$, with probability at least $(1 - \delta_1)$, the teacher is successful with the following bounds:
%Then, in the worst-case for any sample $\widetilde{\Examples}$ and selection of $\widetilde{\hstar}$, the teacher is successful with the following bounds which hold with probability at least $(1 - \delta_1)$:
% Then, for any $\widetilde{\eta}_0$  that satisfies Definition~\ref{def:noise-in-learning-parameters:rate}, there exists a problem setting such that the teacher is unsuccessful w.r.t. criteria C.1 and C.2:
\begin{enumerate}
    \item The learner's error is $O(\epsilon)$ and is bounded as $\Error(\widetilde{\opt}_{\epsilon, \Delta_{\Examples}}) \leq \frac{(\epsilon \cdot Q_{\textnormal{max}}  + \delta_2)}{Q(h^*)}.$
   \item The size of the teaching set is bounded as $|\widetilde{\opt}_{\epsilon, \Delta_{\Examples}}| \leq |\opt_{\hat{\epsilon} }|$ where $\hat{\epsilon} = \frac{(\epsilon \cdot Q_{\textnormal{min}} - \delta_2) \cdot (1 - \eta)^{\lambda \cdot \delta_3}}{Q(h^*)}$
\end{enumerate}
where $Q_{\max} = \max_{h} Q_0(h)$ and $Q_{\min} = \min_{h} Q_0(h)$.
%w.r.t. criteria C.1 and C.2 
%
%\begin{enumerate}
%    \item  For any fixed $\Delta_\eta$ and $\epsilon$, there exists problem settings where $\Error(\widetilde{\opt}_{\epsilon, \Delta_\eta}) \geq \frac{1}{2}$.
%    \item For any fixed constant $\delta$ and $\epsilon$, and any $\hat{\epsilon}$ arbitrarily close to $0$, there exist problem settings where $|\widetilde{\opt}_{\epsilon, \Delta_\eta}| \geq |\opt_{\hat{\epsilon}}|$.
\end{theorem}

%The proof of the theorem is provided in the longer version of the paper (see Footnote~\ref{footnote:longer}). 
Note that the bound is only valid for $\eta < 1$. When $\eta$ approaches 1 and for extreme case $\eta = 1$, the learner reduces to a noise-free version space learner who eliminates all hypothesis immediately. For this setting, bounding the teaching set size requires more combinatorial assumptions on the dataset (e.g., based on separability of data from the hyperplanes)---however, for practical applications, $\eta$ bounded away from $1$ is a more natural setting as analyzed in this theorem.

\subsection{Noise in feature embedding} \label{sec:noise-in-representation:feature}
%Noise in features given by $||\overline{\phi}(x) - \phi(x)||_2 \leq \delta$.

Here, we consider imperfect knowledge in terms of noisy feature map $\widetilde{\phi}$. This is a challenging setting as noise in $\phi$ means error in the predictions of hypotheses $\hypothesis \in \Hypotheses$ which in turn leads to noise in error of hypotheses $\error(.)$ and in the likelihood function $J$.  As noted earlier, the teacher will select a hypothesis $\widetilde{\hstar}$ as the one with minimal error given by $\widetilde{\hstar} \in \argmin_{\hypothesis \in \Hypotheses} \widetilde{\error}(\hypothesis)$.  The following definition quantifies the imperfection in the teacher's knowledge $(Q_0, \eta, \Examples, \widetilde{\hstar}, \widetilde{\phi}, \Hypotheses)$.% and  % and formally quantify the imperfect in teacher's knowledge below:
%. Below, we formally quantify the noise in this setting and specific assumptions.

%\begin{definition}[$\Delta_{\phi}$-imperfect]\label{def:noise-in-representation:feature}
%Let $\Delta_{\phi} = \{\delta_1, \delta_2\}$ for $\delta_1, \delta_2 \geq 0$. We say that teacher's knowledge is $\Delta_{\phi}$-imperfect if the following holds:
%\begin{gather}
%\forall \ \instance \in \Instances, ||\phi(x) - \widetilde{\phi}(x)||_2 \leq \delta_1. \notag\\
%|\widetilde{\textnormal{err}}(h) - \textnormal{err}(h)| \leq \delta_2 \notag
%\end{gather}
%%where a $\delta$-perturbed version of $S$ means that each example $(\instance, \clabel) \in S$ maps uniquely to a perturbed $(\instance', \clabel) \in S'$ such that $||\phi(\instance) - \phi(\instance')||_2 \leq \delta$
%\end{definition}

\begin{definition}[$\Delta_{\phi}$-imperfect]\label{def:noise-in-representation:feature}
Let $\Delta_{\phi} = (\delta_1, \delta_2)$ for $\delta_1, \delta_2 \geq 0$. We say that a teacher's knowledge is $\Delta_{\phi}$-imperfect if the following holds: 
%(i) $\forall \ \instance \in \Instances, ||\phi(x) - \widetilde{\phi}(x)||_2 \leq \delta_1$, and (ii) $\forall \ \hypothesis, |\widetilde{\error}(h) - \textnormal{err}(h)| \leq \delta_2$.
\setlength{\leftmargini}{0.62em}
\begin{itemize}
\item $\forall \ \instance \in \Instances, ||\phi(x) - \widetilde{\phi}(x)||_2 \leq \delta_1$,
\item $\forall \ \hypothesis, |\widetilde{\error}(h) - \textnormal{err}(h)| \leq \delta_2$.
\end{itemize}
\setlength{\leftmargini}{2.5em}

\end{definition}

The following theorem quantifies the effectiveness of teaching of a $\Delta_\phi$-imperfect teacher.
% in this setting in terms of the two success criteria C.1 and C.2 discussed in the previous section.

\begin{theorem}
\label{thm:noise-in-representation:feature}
\looseness-1
Fix $\epsilon \geq 0$ and $\Delta_{\phi} = (\delta_1, \delta_2)$ with $\delta_1, \delta_2  \geq 0$. Consider a $\Delta_{\phi}$-imperfect teacher with knowledge $(Q_0, \eta, \Examples, \widetilde{\hstar}, \widetilde{\phi}, \Hypotheses)$.  Assume the problem setting is $\lambda$-smooth for some $\lambda \geq 0$, that $\eta < 1$, and assume that the error $\widetilde{err}(\widetilde{\hstar}) = 0$.  Then, in the worst-case for any observed $\widetilde{\phi}$ and selection of $\widetilde{\hstar}$, the teacher is successful with the following bounds:
%Consider a teacher with knowledge as  $(Q_0, \eta, \Examples, \widetilde{\hstar}, \widetilde{\phi}, \Hypotheses)$ which is $\Delta_{\phi}$-imperfect as per Definition~\ref{def:noise-in-representation:feature} for $\Delta_{\Instances} = \{\delta_1, \delta_2\}$. Furthermore, assume that the problem setting to be $\lambda$-smoooth as per Definition~\ref{def:noise-in-representation:smoothness}, assume $\eta < 1$, and (for simplicity) assume error of $\widetilde{\hstar} = 0$ see Footnote~\ref{footnote:realizable}.  Let $Q_{max} = \max_{h} Q_0(h)$ and $Q_{min} = \min_{h} Q_0(h)$. Then, in the worst-case for any observed $\widetilde{\phi}$ and selection of $\widetilde{\hstar}$, , the teacher is successful w.r.t. criteria C.1 and C.2 with the following bounds:
% Then, for any $\widetilde{\eta}_0$  that satisfies Definition~\ref{def:noise-in-learning-parameters:rate}, there exists a problem setting such that the teacher is unsuccessful w.r.t. criteria C.1 and C.2:
\begin{enumerate}
    \item The learner's error is $O(\epsilon)$ and is bounded as $\Error(\widetilde{\opt}_{\epsilon, \Delta_{\phi}}) \leq \frac{(\epsilon \cdot Q_{\textnormal{max}}  + \delta_2)}{Q(h^*)\cdot (1 - \eta)^{\lambda \cdot \delta_1}}.$
   \item The size of the teaching set is bounded as $|\widetilde{\opt}_{\epsilon, \Delta_{\phi}}| \leq |\opt_{\hat{\epsilon} }|$ where $\hat{\epsilon} = \frac{(\epsilon \cdot Q_{\textnormal{min}} - \delta_2) \cdot (1 - \eta)^{\lambda \cdot  \delta_1}}{Q(h^*)}$.
\end{enumerate}
%where $Q_{\max} = \max_{h} Q_0(h)$ and $Q_{\min} = \min_{h} Q_0(h)$.
\end{theorem}
%Note that the bound is only for $\eta < 1$ (see, discussion at the end of Theorem~\ref{thm:noise-in-representation:sampleX2}).
In comparison to the error bound in Theorem~\ref{thm:noise-in-representation:sampleX2}, the error bound here with noise in $\phi$ is much worse---this is a lot more challenging setting given that hypotheses predictions on examples can be wrong in this setting. Here, for simplicity of the proof and presentation of results, we assumed that there exists some $\widetilde{h^*}$ for which error  in teacher's representation is $0$, i.e., $\widetilde{\error}(\widetilde{h^*}) = 0$, see discussion in Footnote~\ref{footnote:realizable}. The theorem suggests that when considering additional structural/smoothness assumptions on the problem, the teaching with imperfect knowledge about representations is robust w.r.t. both $M.1$ and $M.2$ success criteria. As we shall see in experiments, these robustness guarantees indeed hold in practice given that the real-world problem settings often respect these regularity assumptions.

%% file: 6_experiments.tex
\begin{figure*}[t!]
\centering
	\begin{subfigure}[b]{0.24\textwidth}
	   \centering
		\includegraphics[width=1\linewidth]{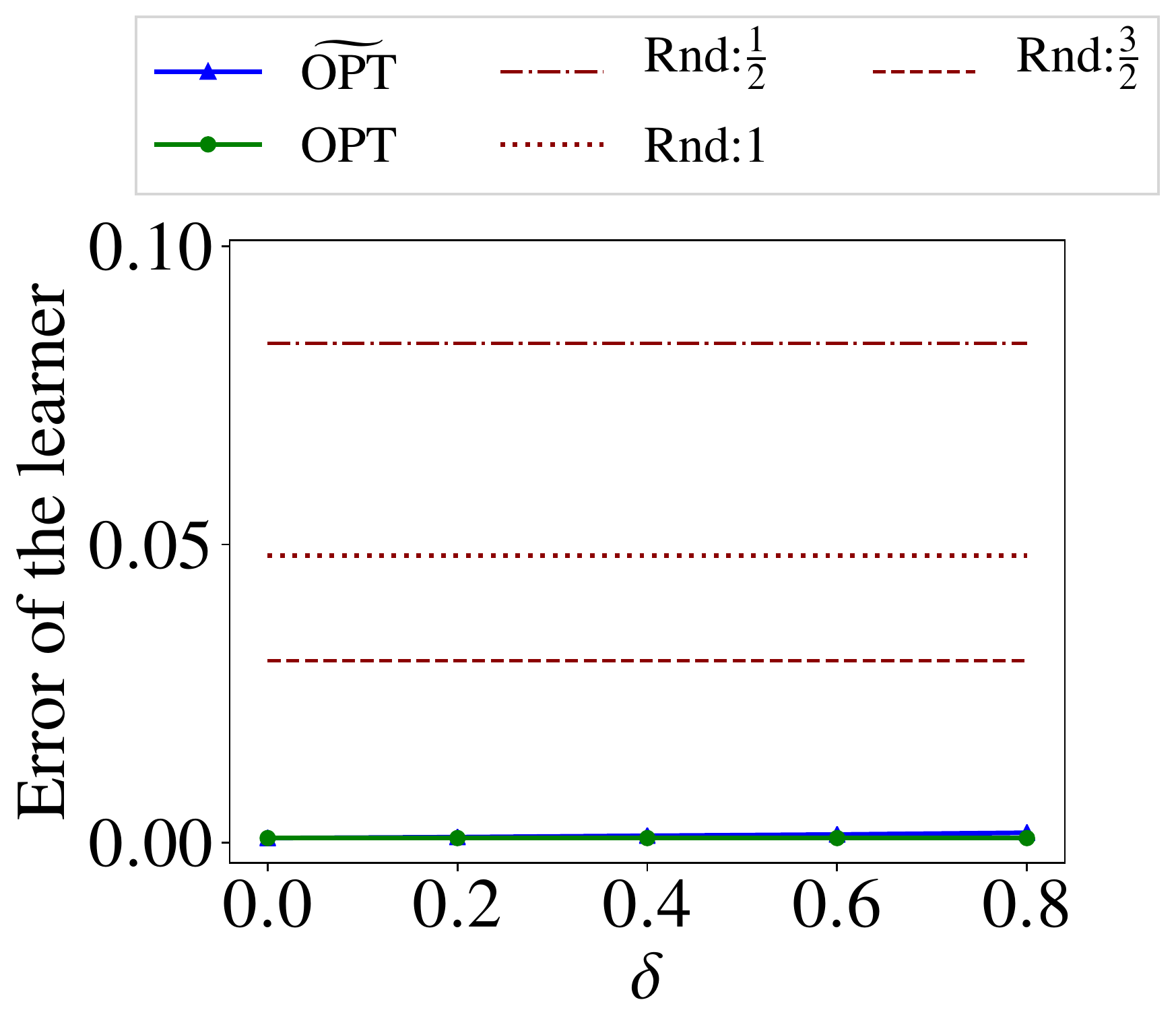}
		\vspace{-5mm}
		\caption{Learner's error: $\widetilde{Q_0}$}
		\label{fig:results.1}
	\end{subfigure}
	\quad
	\begin{subfigure}[b]{0.24\textwidth}
	    \centering
		\includegraphics[width=1\linewidth]{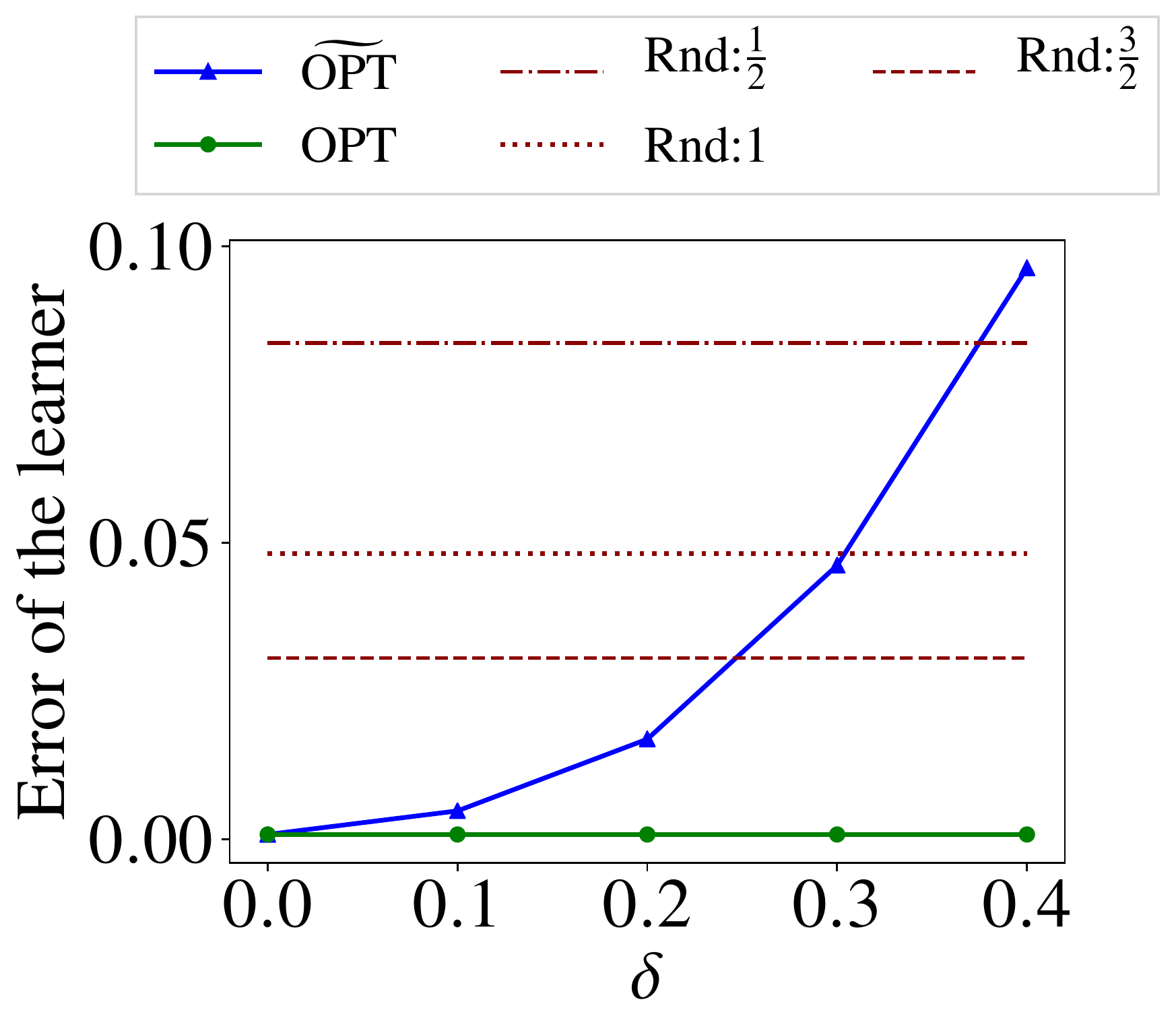}
		\vspace{-5mm}
		\caption{Learner's error: $\widetilde{\eta} = \eta + \delta$}
		\label{fig:results.2}
	\end{subfigure}		
	\begin{subfigure}[b]{0.24\textwidth}
	    \centering
		\includegraphics[width=1\linewidth]{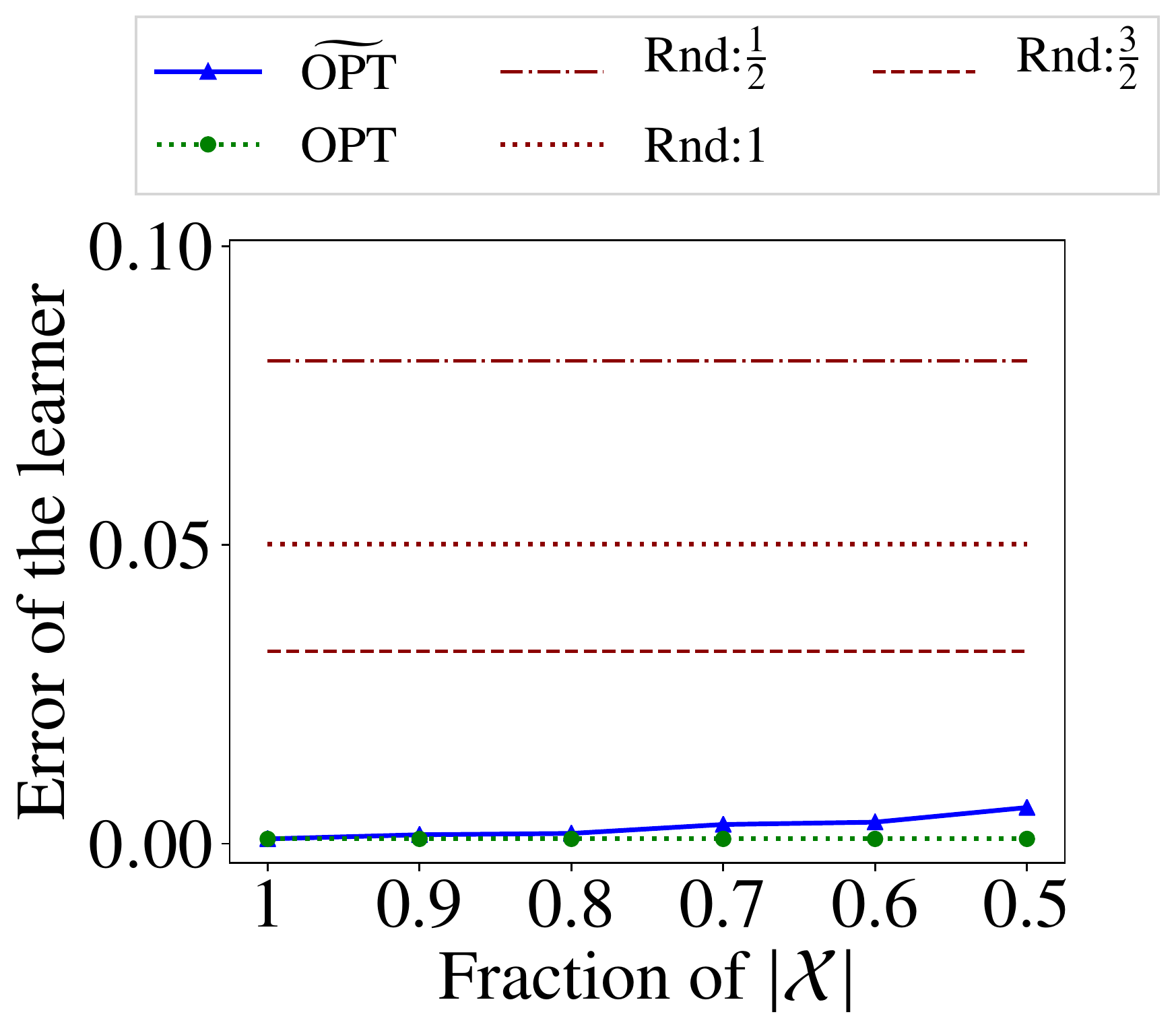}
		\vspace{-5mm}
		\caption{Learner's error: $\widetilde{\Examples}$}
		\label{fig:results.3}
	\end{subfigure}		
	\begin{subfigure}[b]{0.24\textwidth}
	    \centering
		\includegraphics[width=1\linewidth]{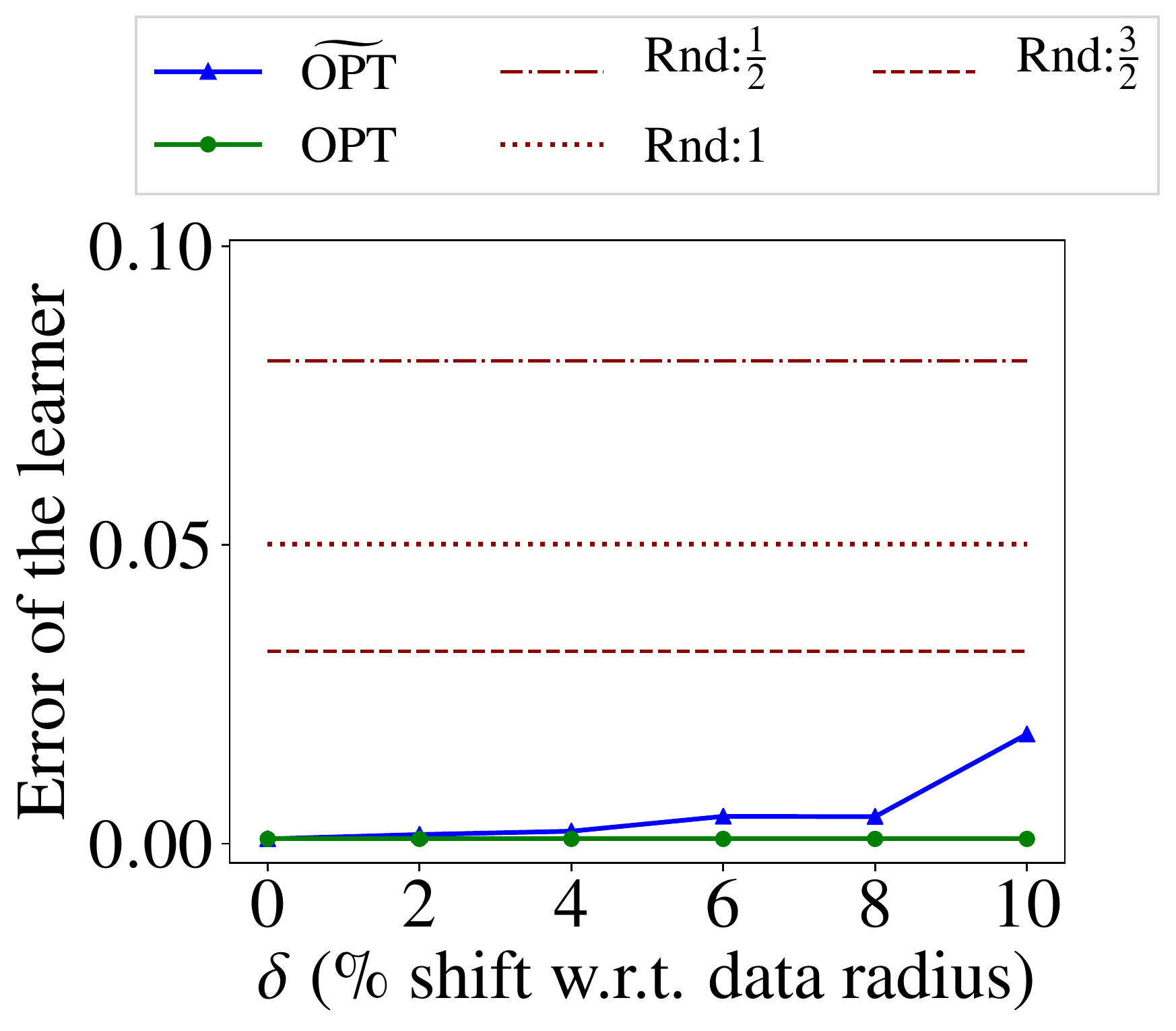}
		\vspace{-5mm}
		\caption{Learner's error: $\widetilde{\phi}$}
		\label{fig:results.4}
	\end{subfigure}	
	%\vspace{10mm}
	\begin{subfigure}[b]{0.24\textwidth}
	   \centering
		\includegraphics[width=1\linewidth]{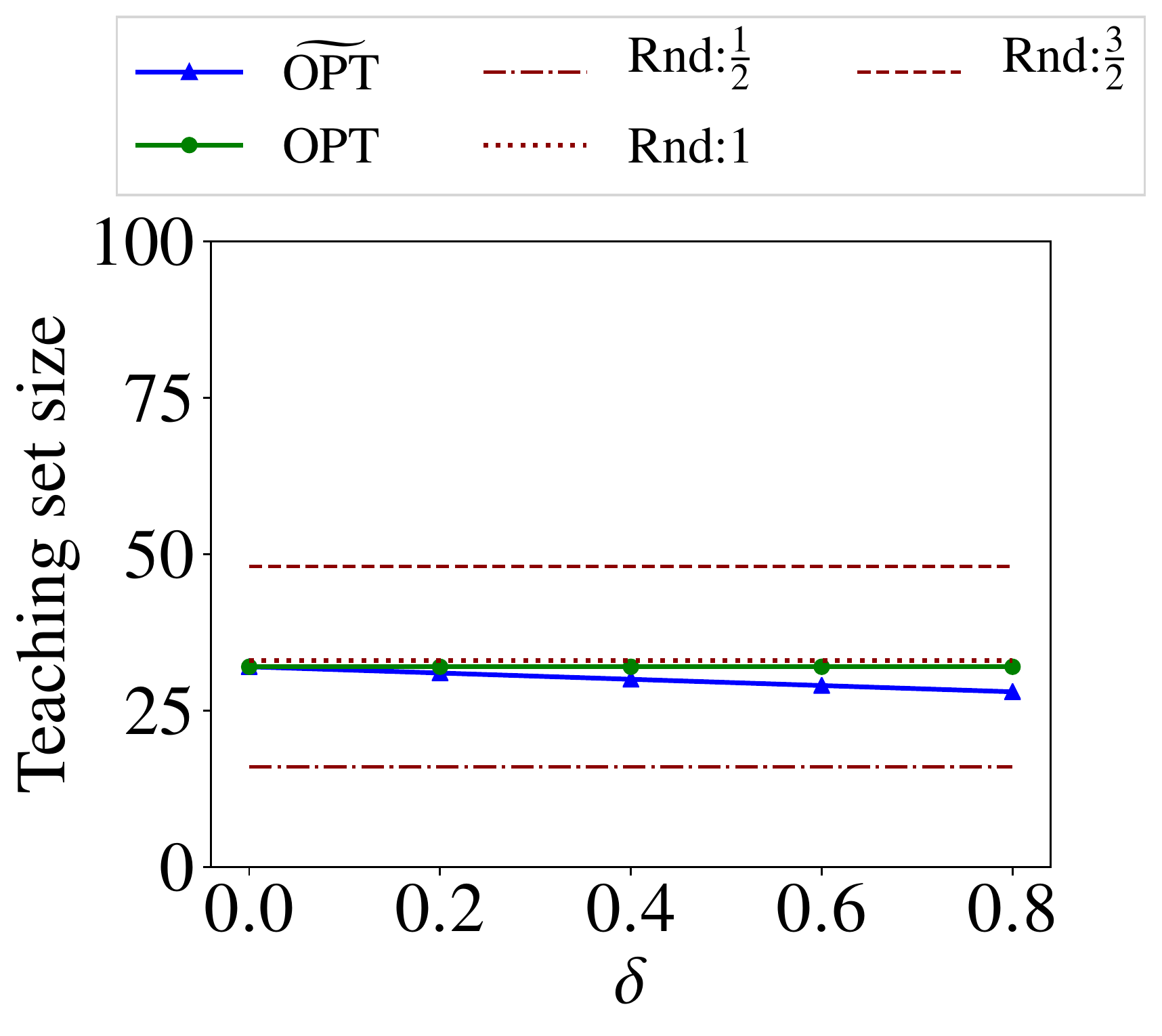}
		\vspace{-5mm}
		\caption{Teaching set size: $\widetilde{Q_0}$}
		\label{fig:results.5}
	\end{subfigure}
	\quad
	\begin{subfigure}[b]{0.24\textwidth}
	    \centering
		\includegraphics[width=1\linewidth]{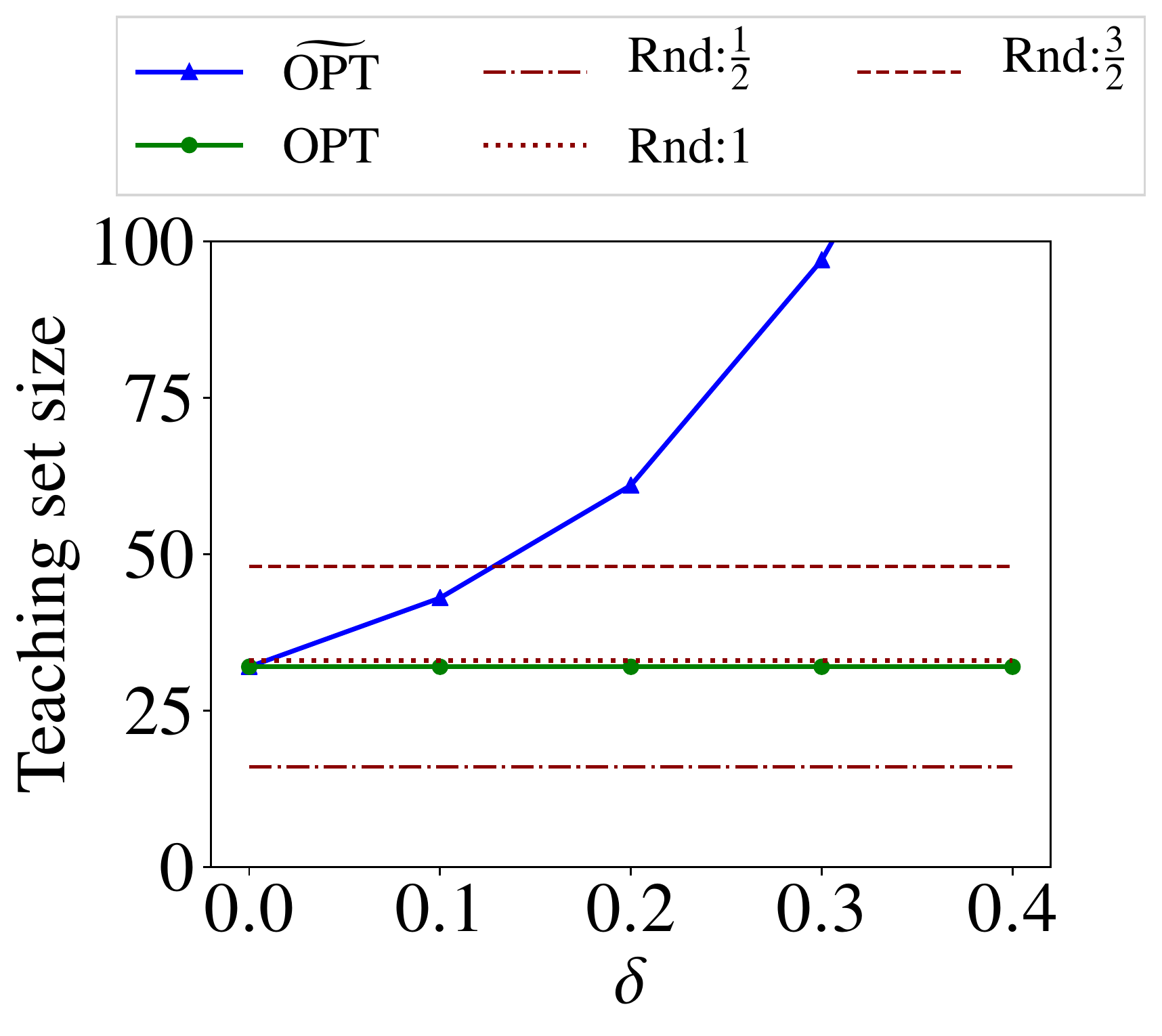}
		\vspace{-5mm}
		\caption{Teaching set size: $\widetilde{\eta} = \eta - \delta$}
		\label{fig:results.6}
	\end{subfigure}		
	\begin{subfigure}[b]{0.24\textwidth}
	    \centering
		\includegraphics[width=1\linewidth]{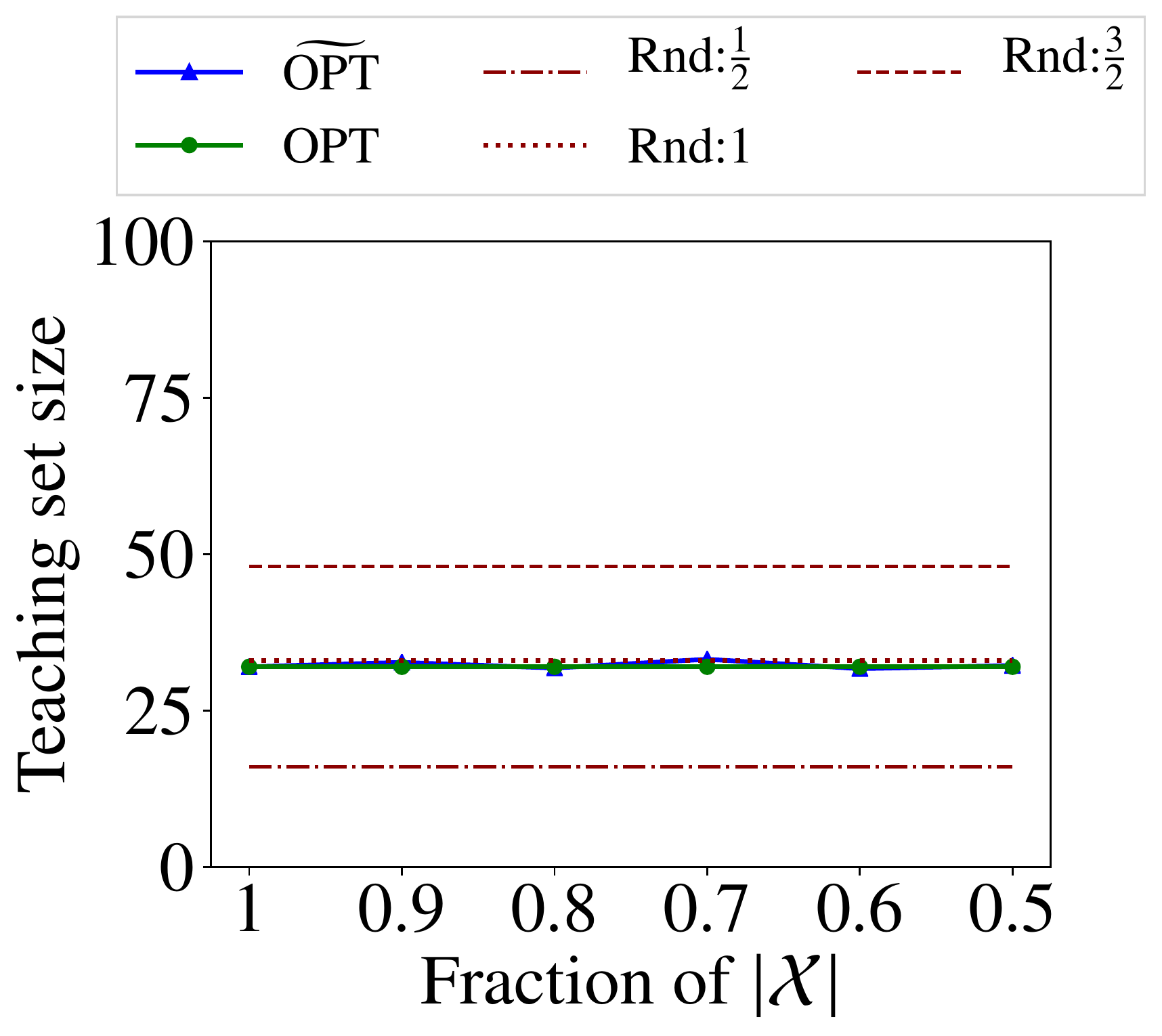}
		\vspace{-5mm}
		\caption{Teaching set size: $\widetilde{\Examples}$}
		\label{fig:results.7}
	\end{subfigure}		
	\begin{subfigure}[b]{0.24\textwidth}
	    \centering
		\includegraphics[width=1\linewidth]{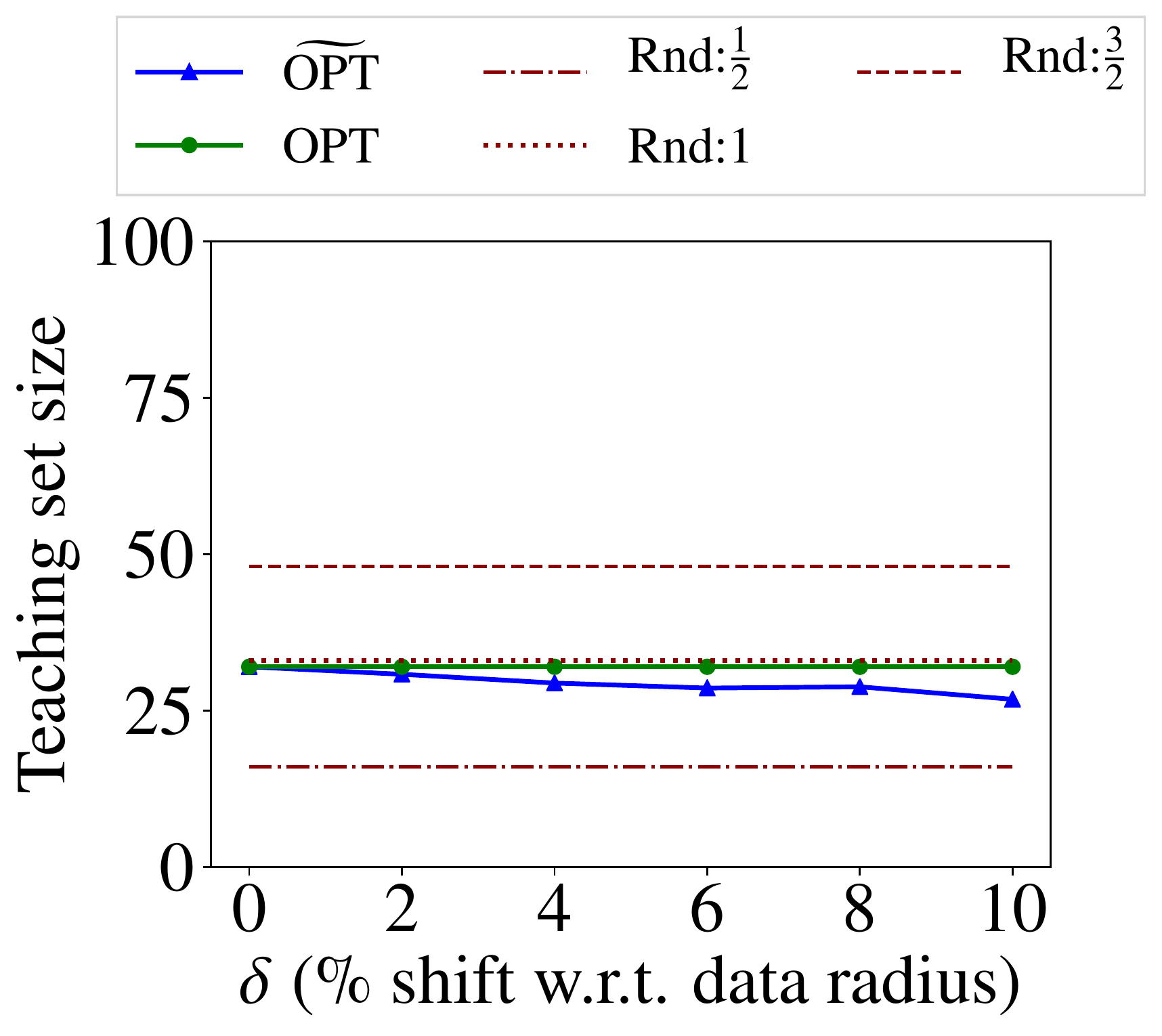}
		\vspace{-5mm}
		\caption{Teaching set size: $\widetilde{\phi}$}
		\label{fig:results.8}
	\end{subfigure}			
	\vspace{-2mm}
   \caption{Experimental results for the problem in Figure~\ref{fig:environments.butterflymoth}. \textbf{(a, e)}: robustness of teaching for $\Delta_{Q_0}$-imperfect teacher. \textbf{(b, f)}: for $\Delta_{\eta}$-imperfect teacher, the learner's error could be high when the teacher overestimates $\eta$ or the teaching set size could be arbitrary large when the teacher underestimates $\eta$. \textbf{(c, g)}: robustness of teaching for $\Delta_{\Examples}$-imperfect teacher. \textbf{(d, h)}: robustness of teaching for $\Delta_{\phi}$-imperfect teacher.%See main text for detailed description.
   % where $\delta$ on x-axis is the     norm of a noise vector in $\R^2$ added to $\phi(x) \ \forall \instance \in \Instances$ \$ $\lVert\widetilde{\phi}(x) - \phi(x)\rVert$ ($\%$ of data spread % which is about $3$ units \yuxin{I don't get this line}, 
   %see Figure~\ref{fig:environments.butterflymoth}).
   }
	\label{fig:results}
	%\vspace{-4mm}
\end{figure*}
%
%
%%%%%%%%%%%%%%%%%%%%%%%%%%%%%%%%%%%%%

\section{Experimental Evaluation}\label{sec:experiments}
In this section, we perform empirical studies to validate the guarantees provided by our theorems, and to showcase that the data regularity assumptions we made in the previous section are satisfied in real-world problem settings.

%\subsection{Teaching task}
%---an important skill required for  biodiversity monitoring related citizen-science projects~\cite{sullivan2009ebird,van2018inaturalist}. 
\paragraph{Teaching task.} We consider a binary image classification task for identifying animal species. This specific task has been studied extensively in the machine teaching literature (see \cite{singla2014near,chen18explain,mac2018teaching,DBLP:conf/aaai/YeoKSMAFDC19}). First, we state the problem setup from the viewpoint of a teacher with full knowledge represented as $(Q_0, \eta, \Examples, \hstar, \phi, \Hypotheses)$. Our problem setup is based on the task and dataset that is used in the works of \cite{singla2014near,DBLP:conf/aaai/YeoKSMAFDC19}. The task is to distinguish ``moths" ($-$ labeled class) from ``butterflies" ($+$ labeled class). We have a total of $|\Examples| = 160$ labeled images and the embedding of instances is shown in Figure~\ref{fig:environments.butterflymoth}.  We have $|\Hypotheses|=67$ hypotheses, and a subset of these hypotheses along with $\hstar$ are shown in Figure~\ref{fig:environments.butterflymoth}. We consider $Q_0$ to be uniform distribution over $\Hypotheses$, $\eta = 0.5$, and have desired $\epsilon=0.001$.\

%\paragraph{Teaching task and data set}
%..\\
%..\\
%..\\
%..\\
%..\\
%..\\
%..\\
%..\\
%..\\
%..\\
%..\\
%..\\
%..\\
%..\\
%
%\paragraph{Parameters and settings}
%..\\
%..\\
%..\\
%..\\
%..\\
%..\\

%\subsection{Empirical results}
\paragraph{Metrics and baselines.} All the results corresponding to four different notions of imperfect teacher are shown in Figure~\ref{fig:results}, averaged over 10 runs. For performance metrics, we plot the eventual error of the learner and the size of the teaching set. In addition to $\opt_\epsilon$ (simply denoted as $\opt$ in plots) and $\apx_{\epsilon, \Delta}$ (simply denoted as $\apx$ in plots), we also have three more baselines denoted as $\textnormal{Rnd:}\frac{1}{2}$, $\textnormal{Rnd:}1$, and $\textnormal{Rnd:}\frac{3}{2}$. These three baselines correspond to teachers who select examples randomly, with set sizes being $\frac{1}{2}$, $1$, and $\frac{3}{2}$ times that of $|\opt|$. 

%First, we consider a teacher who has imperfect knowledge about learning dynamics as studied in Section~\ref{sec:noise-in-learning-parameters}. 
    \paragraph{Empirical results.}  We consider $\Delta_{Q_0}$-imperfect teacher with $\Delta_{Q_0} = (\delta, \delta)$ (i.e., $\delta_1 = \delta_2 = \delta$) with $\delta \in [0, 0.8]$; results are shown in Figures~\ref{fig:results.1},\ref{fig:results.5}. For $\Delta_{\eta}$-imperfect teacher, we vary $\delta \in [0, 0.4]$ considering a teacher who overestimates or underestimates the learning rate; results are shown in Figures~\ref{fig:results.2},\ref{fig:results.6}.  For $\Delta_{\Examples}$-imperfect teacher, we vary the fraction of instances $\Instances$ from $1$ to $0.5$ that we sample to construct $\widetilde{\Examples}$ and sampling is done \emph{i.i.d.}; the performance of this teacher is shown in Figures~\ref{fig:results.3},\ref{fig:results.7}. For $\Delta_{\phi}$-imperfect teacher, we computed noisy representation $\widetilde{\phi}$ by adding a random vector in $\R^2$ of norm $\delta$ as noise to $\phi(\instance) \ \forall \instance \in \Instances$; results are shown in Figures~\ref{fig:results.4},\ref{fig:results.8}. Note that in Figures~\ref{fig:results.4},\ref{fig:results.8}, the norm $\delta$ is shown as relative $\%$ shift w.r.t. data radius, where the radius is $\max_{\instance \in \Instances} ||\phi(x)||_2$ (see Figure~\ref{fig:environments.butterflymoth}).

\looseness-1The results in these plots validate the performance guarantees that we proved in previous sections. It is important to note that for $\Delta_{\Examples}$-imperfect and $\Delta_{\phi}$-imperfect teacher, any additional structural assumptions as were needed by  Definitions~\ref{def:noise-in-representation:sampleX},\ref{def:noise-in-representation:feature} and Theorems~\ref{thm:noise-in-representation:sampleX2},\ref{thm:noise-in-representation:feature} are naturally satisfied in real-world problem settings as is evident in the performance plots.
%we didn't put any additional structural assumptions as were needed by  Definitions~\ref{def:noise-in-representation:sampleX},\ref{def:noise-in-representation:feature} and Theorems~\ref{thm:noise-in-representation:sampleX2},\ref{thm:noise-in-representation:feature}--such regularity conditions are naturally satisfied in real-world problem settings.

%Then, we consider a teacher who has imperfect knowledge about task specification as studied in Section~\ref{sec:noise-in-representation}. 
%; the performance of this teacher is shown in Figure~\ref{fig:results.1},\ref{fig:results.5}
%; the performance of this teacher is shown in Figure~\ref{fig:results.2},\ref{fig:results.6}
%
%, Theorem and 
%
%We do not put any additional regularity assumptions as were stated in Definition~\ref{def:noise-in-representation:sampleX} and Theorem~\ref{thm:noise-in-representation:sampleX2} for guarantees, as these regularly assumptions are satisfied naturally in real-world datasets as is evident in the su
%

% f that we sample to construct $\widetilde{\Examples$ and sampling is done \emph{i.i.d.}

%\paragraph{Noise in $Q_0$}
%..\\
%..\\
%..\\
%..\\
%..\\
%..\\
%..\\
%
%\paragraph{Noise in $\eta$}
%..\\
%..\\
%..\\
%..\\
%..\\
%..\\
%..\\
%
%\paragraph{Noise in sampling}
%..\\
%..\\
%..\\
%..\\
%..\\
%..\\
%..\\
%
%
%\paragraph{Noise in embedding}
%..\\
%..\\
%..\\
%..\\
%..\\
%..\\
%..\\

%\annotate{Finish this by end of page 5.75}

%% file: 7_conclusion.tex
% !TEX root =  main.tex
%%%%%%%%%%%%%%%%%%%%%%%%%%%%%%%%%%%%%%%%%%%%%%%%%%%%%%%%%%
%%%%%%%%%%%%%%%%%%%%%%%%%%%%%%%%%%%%%%%%%%%%%%%%%%%%%%%%%%
%\vspace{-2mm}
\section{Conclusions}
We studied the problem of machine teaching when teacher's knowledge is imperfect. We focused on understanding the robustness of a teacher who constructs teaching sets based on its imperfect knowledge. When having imperfect knowledge about the learner model, our results suggest that having a good estimate of the learning rate is a lot more important than learner's initial knowledge. In terms of imperfect knowledge about the task specification, we introduced some regularity assumptions under which the teacher is robust. Our empirical experiments on a real-world teaching problem further validate our theoretical results.  Our findings have important implications in designing teaching algorithms for real-world applications in education.% where the student is a human learner and teacher is bound to have imperfect knowledge.
%We studied the problem of machine teaching when teacher's knowledge is imperfect. In particular, w
%\annotate{Finish this by end of page 6}

%% file: 8.1.1_appendix_noise-in-learning-parameters_prior.tex
% !TEX root =  main.tex
%%%%%%%%%%%%%%%%%%%%%%%%%%%%%%%%%%%%%%%%%%%%%%%%%%%%%%%%%%
%%%%%%%%%%%%%%%%%%%%%%%%%%%%%%%%%%%%%%%%%%%%%%%%%%%%%%%%%%
%\section{Appendix: Imperfect Knowledge about Learning Parameters}\label{appendix:noise-in-learning-p arameters}
\section{Proof of Theorem~\ref{thm:noise-in-learning-parameters:prior}}
We first introduce a lemma which will be useful in proving the theorem. 

\begin{lemma}\label{lem:Q-vs-QT-parameters:prior}
    For any set of examples $S$ we have $Q(\hypothesis | S) \cdot (1 - \delta_1) \leq \widetilde{Q}(\hypothesis | S) \leq Q(\hypothesis | S)\cdot(1 + \delta_2)$.
\end{lemma}
\begin{proof}
The proof of the left side is as following:
\begin{align*}
    \widetilde{Q}(\hypothesis | S) = \widetilde{Q}_0(\hypothesis) \cdot \prod_{s \in S} J(\clabel_s |  \hypothesis, \instance_s, \eta) \geq Q_0(\hypothesis) \cdot \prod_{s \in S} J(\clabel_s, | \hypothesis, \instance_s, \eta) \cdot (1 - \delta_1) = Q(\hypothesis | S) \cdot (1 - \delta_1).
\end{align*}

The proof of the right side is as following:
\begin{align*}
    \widetilde{Q}(\hypothesis | S) = \widetilde{Q}_0(\hypothesis) \cdot \prod_{s \in S} J(\clabel_s |  \hypothesis, \instance_s, \eta) \leq Q_0(\hypothesis)\prod_{\instance \in \examples} J(\clabel_s |  \hypothesis, \instance_s, \eta) \cdot (1 + \delta_2) = Q(\hypothesis | S) \cdot (1 + \delta_2) .
\end{align*}

\end{proof}

\begin{proof} We will prove two parts of the theorem separately below.

\paragraph{Proof of the first part of the theorem\\}
    By substituting $S$ with $\teachOpt{\epsilon, \Delta_{Q_0}}$ in Lemma~\ref{lem:Q-vs-QT-parameters:prior}, we have
    \begin{align*}
       \Error(\teachOpt{\epsilon, \Delta_{Q_0}}) &=  \sum_{\hypothesis \in \Hypotheses}\frac{Q(\hypothesis|\teachOpt{\epsilon, \Delta_{Q_0}})}{\sum_{\hypothesis' \in \Hypotheses}Q(\hypothesis'|\teachOpt{\epsilon, \Delta_{Q_0}})} \error(\hypothesis)  \\
       &\leq \sum_{\hypothesis \in \Hypotheses}\frac{\frac{\widetilde{Q}(\hypothesis|\teachOpt{\epsilon, \Delta_{Q_0}})}{1 - \delta_1}}{\sum_{\hypothesis' \in \Hypotheses}\frac{\widetilde{Q}(\hypothesis'|\teachOpt{\epsilon, \Delta_{Q_0}})}{1 + \delta_2}} \error(\hypothesis)  \\
       &=  \sum_{\hypothesis \in \Hypotheses}\frac{\widetilde{Q}(\hypothesis|\teachOpt{\epsilon, \Delta_{Q_0}})}{\sum_{\hypothesis' \in \Hypotheses}\widetilde{Q}(\hypothesis'|\teachOpt{\epsilon, \Delta_{Q_0}})} \error(\hypothesis) \cdot \Big(\frac{1 + \delta_2}{1 - \delta_1}\Big) \\
       &= \widetilde{\Error}(\teachOpt{\epsilon, \Delta_{Q_0}}) \cdot \Big(\frac{1 + \delta_2}{1 - \delta_1}\Big) \\
       &\leq  \frac{\epsilon \cdot (1 + \delta_2)}{(1 - \delta_1)}
    \end{align*}
where the last step uses the fact that the set $\teachOpt{\epsilon, \Delta_{Q_0}}$ ensures that the learner's error from teacher's view is not more than $\epsilon$ (i.e.,  $\widetilde{\Error}(\teachOpt{\epsilon, \Delta_{Q_0}}) \leq \epsilon$).
%    Where last inequality is based on the fact that \eqref{eq:F-opt} implies $\Error$ is less than $\epsilon$.

\paragraph{Proof of the second part of the theorem\\}
By definition, $\teachOpt{\epsilon, \Delta_{Q_0}}$ is the smallest set $S$ that ensures
\begin{align}\label{eq:optT-def-parameters:prior}
    \sum_{\hypothesis \in \Hypotheses} \big(\widetilde{Q}_0(\hypothesis) - \widetilde{Q}(\hypothesis|S)\big) \cdot \error(\hypothesis) \geq \sum_{\hypothesis \in \Hypotheses} \widetilde{Q}_0(\hypothesis) \cdot \error(\hypothesis) - \epsilon \cdot \widetilde{Q}_0(\hstar).
\end{align}

In the following, we will use $\hat{\epsilon} = \frac{\epsilon \cdot (1 - \delta_1)}{1 + \delta_2}$. We will show that the set $\opt_{\hat{\epsilon}}$ also satisfies condition in Eq.~\ref{eq:optT-def-parameters:prior}, which in turn would imply that $|\teachOpt{\epsilon, \Delta_{Q_0}}| \leq |\opt_{\hat{\epsilon}}|$. We will make use of the following condition which follows from the definition of  the set $\opt_{\hat{\epsilon}}$:
%Also we know for set $\opt_{\epsilon(\frac{1 - \delta_1}{1 + \delta_2})}$ the following holds
\begin{align}
	%\label{eq:opt-def-parameters:prior}
    \sum_{\hypothesis \in \Hypotheses} \big(Q_0(\hypothesis) - Q(\hypothesis| \opt_{\hat{\epsilon}} )\big) \cdot \error(\hypothesis) &\geq \sum_{\hypothesis \in \Hypotheses} Q_0(\hypothesis) \cdot  \error(\hypothesis) - \frac{\epsilon \cdot (1 - \delta_1)}{(1 + \delta_2)} \cdot Q_0(\hstar) \notag \\
\Rightarrow\quad \sum_{\hypothesis \in \Hypotheses} Q(\hypothesis| \opt_{\hat{\epsilon}} ) \cdot \error(\hypothesis) &\leq \frac{\epsilon \cdot (1 - \delta_1)}{(1 + \delta_2)} \cdot Q_0(\hstar) \label{eq:opt-reduce-parameters:prior}.
\end{align}

%Or equivalently
%\begin{align}\label{eq:opt-reduce-parameters:prior}
%    \sum_{\hypothesis \in \Hypotheses} Q(\hypothesis|S) \cdot \error(\hypothesis) \leq \epsilon \cdot \frac{1 - \delta_1}{1 + \delta_2} \cdot Q(\hstar).
%\end{align}

Substituting $S$ with $\opt_{\hat{\epsilon}}$, we start from the left hand side of Eq.~\ref{eq:optT-def-parameters:prior}, and then we will follow a series of steps to arrive at the right hand side of Eq.~\ref{eq:optT-def-parameters:prior}.

\begin{align*}
 \sum_{\hypothesis \in \Hypotheses} \big(\widetilde{Q}_0(\hypothesis) - \widetilde{Q}(\hypothesis|\opt_{\hat{\epsilon}})\big) \cdot \error(\hypothesis) & \stackrel{(a)}{\geq}  \sum_{\hypothesis \in \Hypotheses} \widetilde{Q}_0(\hypothesis) \cdot \error(\hypothesis) - (1 + \delta_2) \cdot  \sum_{\hypothesis \in \Hypotheses}  Q(\hypothesis|\opt_{\hat{\epsilon}}) \cdot \error(\hypothesis) \\ 
& \stackrel{(b)}{\geq} \sum_{\hypothesis \in \Hypotheses} \widetilde{Q}_0(\hypothesis) \cdot \error(\hypothesis) - (1 + \delta_2) \cdot  \frac{\epsilon \cdot (1 - \delta_1)}{(1 + \delta_2)} \cdot Q_0(\hstar) \\ 
& = \sum_{\hypothesis \in \Hypotheses} \widetilde{Q}_0(\hypothesis) \cdot \error(\hypothesis) -  \epsilon \cdot (1 - \delta_1) \cdot Q_0(\hstar) \\ 
&\stackrel{(c)}{\geq} \sum_{\hypothesis \in \Hypotheses} \widetilde{Q}_0(\hypothesis) \cdot \error(\hypothesis) -  \epsilon \cdot \widetilde{Q}_0(\hstar).
\end{align*}
In the above, steps (a) and (c) follow from Lemma~\ref{lem:Q-vs-QT-parameters:prior}, and step (b) follows by utilizing condition in Eq.~\ref{eq:opt-reduce-parameters:prior}.
\end{proof}

%% file: 8.1.2_appendix_noise-in-learning-parameters_learningrate.tex
% !TEX root =  main.tex
%%%%%%%%%%%%%%%%%%%%%%%%%%%%%%%%%%%%%%%%%%%%%%%%%%%%%%%%%%
%%%%%%%%%%%%%%%%%%%%%%%%%%%%%%%%%%%%%%%%%%%%%%%%%%%%%%%%%%

\section{Proof of Theorem~\ref{thm:noise-in-learning-parameters:learningrate}}

\begin{proof}
Consider a problem setting with $\Hypotheses = \{\bar{h}, \hstar \}$ such that $\error(\bar{h}) = 1$ and $\error(\hstar) = 0$. The values of parameters $\epsilon$ and $\Delta_\eta = (\delta)$ are fixed. We study two separate cases to prove two parts of the theorem.

\paragraph{Proof of the first part of the theorem: $\widetilde{\eta} = \min\{\eta+\delta, 1\}$\\}
Pick $\eta$ such that $\eta, \widetilde{\eta} \in (0, 1)$.  Also, fix a number $k$ such that $$k = \bigg\lceil \frac{\log(\frac{1}{\epsilon})}{\log(\frac{1 - \eta}{1 - \widetilde{\eta}})} \bigg\rceil.$$

Furthermore, we set $Q_0$ so that the following two conditions hold:
\begin{align*}
\frac{Q_0(\bar{h})}{Q_0(\hstar)} = \frac{\epsilon}{(1 - \widetilde{\eta})^k}, \qquad Q_0(\hstar) + Q_0(\bar{h}) = 1\\
\end{align*}

From these parameter settings, it is easy to show that $|\teachOpt{\epsilon, \Delta_\eta}| = k$. Next, we write the learner's error $\Error(\teachOpt{\epsilon, \Delta_{\eta}})$ as follows:
    \begin{align*}
       \Error(\teachOpt{\epsilon, \Delta_{\eta}}) = &  \sum_{\hypothesis \in \Hypotheses}\frac{Q(\hypothesis|\teachOpt{\epsilon, \Delta_{\eta}})}{\sum_{\hypothesis' \in \Hypotheses}Q(\hypothesis'|\teachOpt{\epsilon, \Delta_{\eta}})} \error(\hypothesis) \\
       & =  \frac{Q(\bar{h}|\teachOpt{\epsilon, \Delta_{\eta}})}{Q(\bar{h}|\teachOpt{\epsilon, \Delta_{\eta}}) + Q(\hstar|\teachOpt{\epsilon, \Delta_{\eta}})} \\
       &= \frac{Q_0(\bar{h})\cdot(1-\eta)^k}{Q_0(\bar{h})\cdot(1-\eta)^k + Q_0(\hstar)} \\
       &= \frac{1}{1 + \frac{(1-\widetilde{\eta})^k}{\epsilon \cdot (1-\eta)^k}}\\
       &\approx \frac{1}{2}.
    \end{align*}

%(Since $1 - \eta > 1 - \widetilde{\eta}$, this is not infinite). Furthermore, let $Q_0(\hstar) = 1$, and $Q_0(\hypothesis_0) = \frac{3\epsilon}{2 (1 -  \widetilde{\eta})^n}$.

\paragraph{Proof of the second part of the theorem: $\widetilde{\eta} = \max\{\eta-\delta, 0\}$\\}
Pick $\eta$ such that $\eta, \widetilde{\eta} \in (0, 1)$.  Also, pick any $\hat{\epsilon}$ (arbitrarily close to $0$) and then fix a number $k$ such that $$k = \bigg\lceil \frac{\log(\frac{\epsilon}{\hat{\epsilon}})}{\log(\frac{1 - \widetilde{\eta}}{1 - \eta})} \bigg\rceil.$$

Furthermore, we set $Q_0$ so that the following two conditions hold:
\begin{align*}
\frac{Q_0(\bar{h})}{Q_0(\hstar)} = \frac{\epsilon}{(1 - \widetilde{\eta})^k}, \qquad Q_0(\hstar) + Q_0(\bar{h}) = 1\\
\end{align*}

From these parameter settings, it is easy to show that $|\teachOpt{\epsilon, \Delta_\eta}| = k$ and $|\opt_{\hat{\epsilon}}| = k$. This in turn proves the desired statement that $|\teachOpt{\epsilon, \Delta_\eta}| \geq |\opt_{\hat{\epsilon}}|$ for $\hat{\epsilon}$ arbitrarily close to $0$.
\end{proof}

%% file: 8.2.1_appendix_noise-in-representation-sampleX.tex
\section{Proof of Theorem~\ref{thm:noise-in-representation:sampleX2}}

\begin{proof} We will prove the two parts of the theorem separately below.

\paragraph{Proof of the first part of the theorem\\}
By definition, $\teachOpt{\epsilon, \Delta_{\Examples}}$ satisfies the following: 
\begin{align*}
    \sum_{\hypothesis \in \Hypotheses} \big(Q_0(\hypothesis) - Q(\hypothesis|\teachOpt{\epsilon, \Delta_{\Examples}})\big) \cdot \widetilde{\error}(\hypothesis) & \geq \sum_{\hypothesis \in \Hypotheses} Q_0(\hypothesis) \cdot \widetilde{\error}(\hypothesis) - \epsilon \cdot Q_0(\widetilde{\hstar})\\
\Rightarrow\quad \sum_{\hypothesis \in \Hypotheses} Q(\hypothesis|\teachOpt{\epsilon, \Delta_{\Examples}}) \cdot \widetilde{\error}(\hypothesis) & \leq \epsilon \cdot Q_0(\widetilde{\hstar}) \\
\end{align*}

This in turn implies that, with probability at least $(1 - \delta_1)$, the set $\teachOpt{\epsilon, \Delta_{\Examples}}$ satisfies the following: 
\begin{align}
\sum_{\hypothesis \in \Hypotheses} Q(\hypothesis|\teachOpt{\epsilon, \Delta_{\Examples}}) \cdot \big(\error(\hypothesis) - \delta_2\big) & \leq \epsilon \cdot Q_0(\widetilde{\hstar}) \notag \\
\Rightarrow\quad \sum_{\hypothesis \in \Hypotheses} Q(\hypothesis|\teachOpt{\epsilon, \Delta_{\Examples}}) \cdot \error(\hypothesis) &\leq  \epsilon \cdot Q_0(\widetilde{\hstar}) + \big(\sum_{\hypothesis \in \Hypotheses} Q(\hypothesis|\teachOpt{\epsilon, \Delta_{\Examples}})\big)\cdot\delta_2 \notag \\
\Rightarrow\quad \sum_{\hypothesis \in \Hypotheses} Q(\hypothesis|\teachOpt{\epsilon, \Delta_{\Examples}})\cdot \error(\hypothesis) &\leq  \epsilon \cdot Q_\textnormal{max} + \delta_2 \label{eq.appendix3.label1}
\end{align}
%where the step (a) uses the fact that $\sum_{\hypothesis \in \Hypotheses} Q(\hypothesis|\teachOpt{\epsilon, \Delta_{\Examples}}) \leq 1$.  

Next, we compute the error of the learner for the set $\teachOpt{\epsilon, \Delta_{\Examples}}$  as follows:
    \begin{align*}
       \Error(\teachOpt{\epsilon, \Delta_{\Examples}}) &=  \sum_{\hypothesis \in \Hypotheses}\frac{Q(\hypothesis|\teachOpt{\epsilon, \Delta_{\Examples}})}{\sum_{\hypothesis' \in \Hypotheses}Q(\hypothesis'|\teachOpt{\epsilon, \Delta_{\Examples}})} \error(\hypothesis)  \\
       &\leq \sum_{\hypothesis \in \Hypotheses}\frac{Q(\hypothesis|\teachOpt{\epsilon, \Delta_{\Examples}})}{Q_0(\hstar)} \error(\hypothesis)\\
       &\stackrel{(a)}{\leq} \frac{\epsilon \cdot Q_\textnormal{max} + \delta_2}{Q_0(\hstar)}        \\    
    \end{align*}
where the last step marked as (a) uses the condition from Eq.~\ref{eq.appendix3.label1}.

\paragraph{Proof of the second part of the theorem\\}
In this part, we will make use of the  statement below which holds when the problem setting is $\lambda$-smooth as per the theorem assumptions.
%which we state without formal proof. 
%%
For any set of examples $S \subseteq \Examples$ and for a set $S' \subseteq \widetilde{\Examples}$ which is a $\delta_3$-perturbed version of $S$, the following holds for all $\hypothesis \in \Hypotheses$ :
\begin{align}
Q(\hypothesis|S') \leq Q(\hypothesis|S) \cdot (1 - \eta)^{- \lambda \cdot \delta_3} \label{lem:Q-S-S'-representation:sampleX}
\end{align}

We begin the proof by noting the following: By definition, $\teachOpt{\epsilon, \Delta_{\Examples}}$ is the smallest set $S \subseteq \widetilde{\Examples}$ that ensures
\begin{align}\label{eq.appendix3.label2}
    \sum_{\hypothesis \in \Hypotheses} \big(Q_0(\hypothesis) - Q(\hypothesis|S)\big) \cdot \widetilde{\error}(\hypothesis) \geq \sum_{\hypothesis \in \Hypotheses} Q_0(\hypothesis) \cdot \widetilde{\error}(\hypothesis) - \epsilon \cdot Q_0(\widetilde{\hstar}).
\end{align}

The proof follows along the same arguments as the proof of Theorem~\ref{thm:noise-in-learning-parameters:prior}, however, the key challenge here is that the set $S$ in Eq.~\ref{eq.appendix3.label2} is selected from $\widetilde{\Examples}$ instead of $\Examples$. In the following, we will use $\hat{\epsilon} = \frac{(\epsilon \cdot Q_{\textnormal{min}} - \delta_2) \cdot (1 - \eta)^{\lambda \cdot \delta_3}}{Q_0(h^*)}$.  We will make use of the following condition which follows from the definition of  the set $\opt_{\hat{\epsilon}}$:
\begin{align}
	%\label{eq:opt-def-parameters:prior}
    \sum_{\hypothesis \in \Hypotheses} \big(Q_0(\hypothesis) - Q(\hypothesis| \opt_{\hat{\epsilon}} )\big) \cdot \error(\hypothesis) &\geq \sum_{\hypothesis \in \Hypotheses} Q_0(\hypothesis) \cdot  \error(\hypothesis) - \hat{\epsilon} \cdot Q_0(\hstar) \notag \\
\Rightarrow\quad \sum_{\hypothesis \in \Hypotheses} Q(\hypothesis| \opt_{\hat{\epsilon}} ) \cdot \error(\hypothesis) &\leq \hat{\epsilon} \cdot Q_0(\hstar) \label{eq.appendix3.label3}.
\end{align}

Based on the second condition of $\Delta_{\Examples}$-imperfect teacher in Definition~\ref{def:noise-in-representation:sampleX}, we know that with probability at least $(1 - \delta_1)$ there exists a set $A \subseteq \widetilde{\Examples}$ which is $\delta_3$-perturbed version of $\opt_{\hat{\epsilon}}$. Note that, here we are making use of the theorem's assumption that $\widetilde{\Examples}$ is sufficiently large, i.e. $|\widetilde{\Examples}| \geq |\opt_{\hat{\epsilon}}|$.  Then, we take the condition in Eq.~\ref{lem:Q-S-S'-representation:sampleX} (by substituting $S$ with $\opt_{\hat{\epsilon}}$ and $S'$ with $A$) and condition in Eq.~\ref{eq.appendix3.label3} to derive the following which holds with probability at least $(1 - \delta_1)$:
\begin{align}
\sum_{\hypothesis \in \Hypotheses} Q(\hypothesis| A ) \cdot \error(\hypothesis) &\leq (1 - \eta)^{- \lambda \cdot \delta_3} \cdot \hat{\epsilon} \cdot Q_0(\hstar) \\
\Rightarrow\quad \sum_{\hypothesis \in \Hypotheses} Q(\hypothesis| A ) \cdot \error(\hypothesis) &\leq \epsilon \cdot Q_\textnormal{min} - \delta_2  \label{eq.appendix3.label4}.
\end{align}

Next, we will show that the set $A$ also satisfies condition Eq.~\ref{eq.appendix3.label2}, which in turn would imply that $|\teachOpt{\epsilon, \Delta_{\Examples}}| \leq |A|  = |\opt_{\hat{\epsilon}}|$.   Substituting $S$ with $A$, we start from the left hand side of Eq.~\ref{eq.appendix3.label2}, and then we will follow a series of steps to arrive at the right hand side of the Eq.~\ref{eq.appendix3.label2}. Note that the results we are proving below holds with probability at least $(1 - \delta_1)$.
% (i.e., the event that the assumptions of theorem statement are true).

\begin{align*}
 \sum_{\hypothesis \in \Hypotheses} \big(Q_0(\hypothesis) - Q(\hypothesis|A)\big) \cdot \widetilde{\error}(\hypothesis) & \stackrel{(a)}{\geq}  \sum_{\hypothesis \in \Hypotheses} Q_0(\hypothesis) \cdot \widetilde{\error}(\hypothesis) -  \sum_{\hypothesis \in \Hypotheses}  Q(\hypothesis|A) \cdot \big(\error(\hypothesis) + \delta_2\big)\\ 
& \stackrel{(b)}{\geq}  \sum_{\hypothesis \in \Hypotheses} Q_0(\hypothesis) \cdot \widetilde{\error}(\hypothesis) -   \delta_2 - \sum_{\hypothesis \in \Hypotheses}  Q(\hypothesis|A) \cdot \error(\hypothesis)\\ 
& \stackrel{(c)}{\geq}  \sum_{\hypothesis \in \Hypotheses} Q_0(\hypothesis) \cdot \widetilde{\error}(\hypothesis) -   \delta_2 - \epsilon \cdot Q_\textnormal{min} + \delta_2\\ 
%& \geq  \sum_{\hypothesis \in \Hypotheses} Q_0(\hypothesis) \cdot \widetilde{\error}(\hypothesis) - \epsilon \cdot Q_\textnormal{min} \\ 
& \geq  \sum_{\hypothesis \in \Hypotheses} Q_0(\hypothesis) \cdot \widetilde{\error}(\hypothesis) - \epsilon \cdot Q_0(\widetilde{\hstar}). \\ 
\end{align*}
In the above, step (a) uses the first condition from Definition~\ref{def:noise-in-representation:sampleX}, step (b) uses the fact that  $\sum_{\hypothesis \in \Hypotheses}  Q(\hypothesis|A) \leq 1$, and step (c) follows by utilizing condition in Eq.~\ref{eq.appendix3.label4}.
\end{proof}

%% file: 8.2.2_appendix_noise-in-representation-embedding.tex
\section{Proof of Theorem~\ref{thm:noise-in-representation:feature}}

\begin{proof}  Based on the first condition of $\Delta_\phi$-imperfect teacher in Definition~\ref{def:noise-in-representation:feature}, we know that for any set of examples $S$, instances with $\widetilde{\phi}(.)$ feature map are $\delta_1$-perturbed version of instances with $\phi(.)$ feature map. Therefore, since the setting is $\lambda$-smooth, number of instances $\instance \in S$ for which $\hypothesis(\instance; \phi) \neq \hypothesis(\instance; \widetilde{\phi})$ is less than $\lambda \cdot \delta_1$.  Using this observation, we have the following  statement below, similar in spirit of statement in Eq.~\ref{lem:Q-S-S'-representation:sampleX}.  For any set of examples $S \subseteq \Examples$, the following holds for all $\hypothesis \in \Hypotheses$ when the problem setting is $\lambda$-smooth as per the theorem conditions:
\begin{align}
\widetilde{Q}(\hypothesis|S) \leq Q(\hypothesis|S) \cdot (1 - \eta)^{-\lambda \cdot \delta_1} \label{eq.appendix4.label1} \\
Q(\hypothesis|S) \leq \widetilde{Q}(\hypothesis|S) \cdot (1 - \eta)^{-\lambda \cdot \delta_1}  \notag
\end{align}

Next, we will prove two parts of the theorem separately below. 

\paragraph{Proof of the first part of the theorem\\}
By definition, $\teachOpt{\epsilon, \Delta_{\phi}}$ satisfies the following: 
\begin{align*}
    \sum_{\hypothesis \in \Hypotheses} \big(Q_0(\hypothesis) - \widetilde{Q}(\hypothesis|\teachOpt{\epsilon, \Delta_{\phi}})\big) \cdot \widetilde{\error}(\hypothesis) & \geq \sum_{\hypothesis \in \Hypotheses} Q_0(\hypothesis) \cdot \widetilde{\error}(\hypothesis) - \epsilon \cdot Q_0(\widetilde{\hstar})\\
\Rightarrow\quad \sum_{\hypothesis \in \Hypotheses} \widetilde{Q}(\hypothesis|\teachOpt{\epsilon, \Delta_{\phi}}) \cdot \widetilde{\error}(\hypothesis) & \leq \epsilon \cdot Q_0(\widetilde{\hstar}) \\
\end{align*}

This in turn implies that the set $\teachOpt{\epsilon, \Delta_{\phi}}$ satisfies the following: 
\begin{align}
\sum_{\hypothesis \in \Hypotheses} \widetilde{Q}(\hypothesis|\teachOpt{\epsilon, \Delta_{\phi}}) \cdot \big(\error(\hypothesis) - \delta_2\big) & \leq \epsilon \cdot Q_0(\widetilde{\hstar}) \notag \\
\Rightarrow\quad \sum_{\hypothesis \in \Hypotheses} \widetilde{Q}(\hypothesis|\teachOpt{\epsilon, \Delta_{\phi}}) \cdot \error(\hypothesis) &\leq  \epsilon \cdot Q_0(\widetilde{\hstar}) + \big(\sum_{\hypothesis \in \Hypotheses} \widetilde{Q}(\hypothesis|\teachOpt{\epsilon, \Delta_{\phi}})\big)\cdot\delta_2 \notag \\
\stackrel{(a)}{\Rightarrow}\quad \sum_{\hypothesis \in \Hypotheses} \widetilde{Q}(\hypothesis|\teachOpt{\epsilon, \Delta_{\phi}})\cdot \error(\hypothesis) &\leq  \epsilon \cdot Q_\textnormal{max} + \delta_2 \notag \\
\stackrel{(b)}{\Rightarrow}\quad \sum_{\hypothesis \in \Hypotheses} Q(\hypothesis|\teachOpt{\epsilon, \Delta_{\phi}})\cdot \error(\hypothesis) &\leq  \frac{\epsilon \cdot Q_\textnormal{max} + \delta_2}{(1 - \eta)^{\lambda \cdot \delta_1} }
\label{eq.appendix4.label2}
\end{align}
where step (a) uses the fact that $\sum_{\hypothesis \in \Hypotheses} \widetilde{Q}(\hypothesis|\teachOpt{\epsilon, \Delta_{\phi}}) \leq 1$, and step (b) uses the condition from Eq.~\ref{eq.appendix4.label1}. 

Next, we compute the error of the learner for set $\teachOpt{\epsilon, \Delta_{\phi}}$  as follows:
    \begin{align*}
       \Error(\teachOpt{\epsilon, \Delta_{\phi}}) &=  \sum_{\hypothesis \in \Hypotheses}\frac{Q(\hypothesis|\teachOpt{\epsilon, \Delta_{\phi}})}{\sum_{\hypothesis' \in \Hypotheses}Q(\hypothesis'|\teachOpt{\epsilon, \Delta_{\phi}})} \error(\hypothesis)  \\
       &\leq \sum_{\hypothesis \in \Hypotheses}\frac{Q(\hypothesis|\teachOpt{\epsilon, \Delta_{\phi}})}{Q_0(\hstar)} \error(\hypothesis)\\
       &\stackrel{(a)}{\leq} \frac{\epsilon \cdot Q_\textnormal{max} + \delta_2}{Q_0(\hstar)\cdot(1 - \eta)^{\lambda \cdot \delta_1}}        \\    
    \end{align*}
where the last step marked as (a) uses the condition from Eq.~\ref{eq.appendix4.label2}.

\paragraph{Proof of the second part of the theorem\\}
By definition, $\teachOpt{\epsilon, \Delta_{\phi}}$ is the smallest set $S$ that ensures
\begin{align}\label{eq.appendix4.label3}
    \sum_{\hypothesis \in \Hypotheses} \big(Q_0(\hypothesis) - \widetilde{Q}(\hypothesis|S)\big) \cdot \widetilde{\error}(\hypothesis) \geq \sum_{\hypothesis \in \Hypotheses} Q_0(\hypothesis) \cdot \widetilde{\error}(\hypothesis) - \epsilon \cdot Q_0(\widetilde{\hstar}).
\end{align}

In the following, we will use $\hat{\epsilon} = \frac{(\epsilon \cdot Q_{\textnormal{min}} - \delta_2) \cdot (1 - \eta)^{\lambda \cdot \delta_1}}{Q_0(h^*)}$. We will show that the set $\opt_{\hat{\epsilon}}$ also satisfies condition Eq.~\ref{eq.appendix4.label3}, which in turn would imply that $|\teachOpt{\epsilon, \Delta_{\phi}}| \leq |\opt_{\hat{\epsilon}}|$. We will make use of the following condition which follows from the definition of  the set $\opt_{\hat{\epsilon}}$:
\begin{align}
    \sum_{\hypothesis \in \Hypotheses} \big(Q_0(\hypothesis) - Q(\hypothesis| \opt_{\hat{\epsilon}} )\big) \cdot \error(\hypothesis) &\geq \sum_{\hypothesis \in \Hypotheses} Q_0(\hypothesis) \cdot  \error(\hypothesis) - \hat{\epsilon} \cdot Q_0(\hstar) \notag \\
\Rightarrow\quad \sum_{\hypothesis \in \Hypotheses} Q(\hypothesis| \opt_{\hat{\epsilon}} ) \cdot \error(\hypothesis) &\leq \hat{\epsilon} \cdot Q_0(\hstar) \label{eq.appendix4.label4}.
\end{align}
%
%%Also we know for set $\opt_{\epsilon(\frac{1 - \delta_1}{1 + \delta_2})}$ the following holds
%\begin{align}
%	%\label{eq:opt-def-parameters:prior}
%    \sum_{\hypothesis \in \Hypotheses} \big(Q_0(\hypothesis) - Q(\hypothesis| \opt_{\hat{\epsilon}} )\big) \cdot \error(\hypothesis) &\geq \sum_{\hypothesis \in \Hypotheses} Q_0(\hypothesis) \cdot  \error(\hypothesis) - \frac{\epsilon \cdot (1 - \delta_1)}{(1 + \delta_2)} \cdot Q_0(\hstar) \notag \\
%\Rightarrow\quad \sum_{\hypothesis \in \Hypotheses} Q(\hypothesis| \opt_{\hat{\epsilon}} ) \cdot \error(\hypothesis) &\leq \frac{\epsilon \cdot (1 - \delta_1)}{(1 + \delta_2)} \cdot Q_0(\hstar) \label{eq:opt-reduce-parameters:prior}.
%\end{align}

%Or equivalently
%\begin{align}\label{eq:opt-reduce-parameters:prior}
%    \sum_{\hypothesis \in \Hypotheses} Q(\hypothesis|S) \cdot \error(\hypothesis) \leq \epsilon \cdot \frac{1 - \delta_1}{1 + \delta_2} \cdot Q(\hstar).
%\end{align}

Substituting $S$ with $\opt_{\hat{\epsilon}}$, we start from the left hand side of Eq.~\ref{eq.appendix4.label3}, and then we will follow a series of steps to arrive at the right hand side of the Eq.~\ref{eq.appendix4.label3}.

\begin{align*}
 \sum_{\hypothesis \in \Hypotheses} \big(Q_0(\hypothesis) - \widetilde{Q}(\hypothesis|\opt_{\hat{\epsilon}})\big) \cdot \widetilde{\error}(\hypothesis) & \stackrel{(a)}{\geq}  \sum_{\hypothesis \in \Hypotheses} Q_0(\hypothesis) \cdot \widetilde{\error}(\hypothesis) -  \sum_{\hypothesis \in \Hypotheses}  \widetilde{Q}(\hypothesis|\opt_{\hat{\epsilon}}) \cdot \big(\error(\hypothesis) + \delta_2\big)\\ 
& \stackrel{(b)}{\geq}  \sum_{\hypothesis \in \Hypotheses} Q_0(\hypothesis) \cdot \widetilde{\error}(\hypothesis) -   \delta_2 - \sum_{\hypothesis \in \Hypotheses}  \widetilde{Q}(\hypothesis|\opt_{\hat{\epsilon}}) \cdot \error(\hypothesis)\\ 
& \stackrel{(c)}{\geq}  \sum_{\hypothesis \in \Hypotheses} Q_0(\hypothesis) \cdot \widetilde{\error}(\hypothesis) -   \delta_2 - (1 - \eta)^{-\lambda \cdot \delta_1} \cdot \sum_{\hypothesis \in \Hypotheses}  Q(\hypothesis|\opt_{\hat{\epsilon}}) \cdot \error(\hypothesis)\\ 
& \stackrel{(d)}{\geq}  \sum_{\hypothesis \in \Hypotheses} Q_0(\hypothesis) \cdot \widetilde{\error}(\hypothesis) -   \epsilon \cdot Q_{\textnormal{min}}\\ 
& \geq  \sum_{\hypothesis \in \Hypotheses} Q_0(\hypothesis) \cdot \widetilde{\error}(\hypothesis) - \epsilon \cdot Q_0(\widetilde{\hstar}). \\ 
\end{align*}
In the above, step (a) uses the second condition from Definition~\ref{def:noise-in-representation:feature}, step (b) uses the fact that  $\sum_{\hypothesis \in \Hypotheses}  \widetilde{Q}(\hypothesis|\opt_{\hat{\epsilon}}) \leq 1$, and steps (c), (d) follow by utilizing conditions in Eq.~\ref{eq.appendix4.label1} and Eq.~\ref{eq.appendix4.label4}.

%\begin{align*}
% \sum_{\hypothesis \in \Hypotheses} \big(Q_0(\hypothesis) - \widetilde{Q}(\hypothesis|\opt_{\hat{\epsilon}})\big) \cdot \widetilde{\error}(\hypothesis) & \stackrel{(a)}{\geq}  \sum_{\hypothesis \in \Hypotheses} \widetilde{Q}_0(\hypothesis) \cdot \error(\hypothesis) - (1 + \delta_2) \cdot  \sum_{\hypothesis \in \Hypotheses}  Q(\hypothesis|\opt_{\hat{\epsilon}}) \cdot \error(\hypothesis) \\ 
%& \stackrel{(b)}{\geq} \sum_{\hypothesis \in \Hypotheses} \widetilde{Q}_0(\hypothesis) \cdot \error(\hypothesis) - (1 + \delta_2) \cdot  \frac{\epsilon \cdot (1 - \delta_1)}{(1 + \delta_2)} \cdot Q_0(\hstar) \\ 
%& = \sum_{\hypothesis \in \Hypotheses} \widetilde{Q}_0(\hypothesis) \cdot \error(\hypothesis) -  \epsilon \cdot (1 - \delta_1) \cdot Q_0(\hstar) \\ 
%&\stackrel{(c)}{\geq} \sum_{\hypothesis \in \Hypotheses} \widetilde{Q}_0(\hypothesis) \cdot \error(\hypothesis) -  \epsilon \cdot \widetilde{Q}_0(\hstar).
%\end{align*}
%In the above, steps (a) and (c) follow from Lemma~\ref{lem:Q-vs-QT-parameters:prior}, and step (b) follows by utilizing conditions in Eq.~\ref{eq.appendix4.label1} and ~\ref{eq.appendix4.label1}.

\end{proof}

%% file: main.bbl
\begin{thebibliography}{}

\bibitem[\protect\citeauthoryear{Archambault \bgroup \em et al.\egroup
  }{2009}]{archambault2009student}
Isabelle Archambault, Michel Janosz, Jean-S{\'e}bastien Fallu, and Linda~S
  Pagani.
\newblock Student engagement and its relationship with early high school
  dropout.
\newblock {\em Journal of adolescence}, 32(3):651--670, 2009.

\bibitem[\protect\citeauthoryear{Chen \bgroup \em et al.\egroup
  }{2018a}]{chen18explain}
Yuxin Chen, Oisin~Mac Aodha, Shihan Su, Pietro Perona, and Yisong Yue.
\newblock Near-optimal machine teaching via explanatory teaching sets.
\newblock In {\em AISTATS}, 2018.

\bibitem[\protect\citeauthoryear{Chen \bgroup \em et al.\egroup
  }{2018b}]{DBLP:conf/nips/ChenSAPY18}
Yuxin Chen, Adish Singla, Oisin {Mac Aodha}, Pietro Perona, and Yisong Yue.
\newblock Understanding the role of adaptivity in machine teaching: The case of
  version space learners.
\newblock In {\em NeurIPS}, pages 1483--1493, 2018.

\bibitem[\protect\citeauthoryear{Dasgupta \bgroup \em et al.\egroup
  }{2019}]{DBLP:conf/icml/Dasgupta0PZ19}
Sanjoy Dasgupta, Daniel Hsu, Stefanos Poulis, and Xiaojin Zhu.
\newblock Teaching a black-box learner.
\newblock In {\em ICML}, pages 1547--1555, 2019.

\bibitem[\protect\citeauthoryear{Goldman and
  Kearns}{1995}]{goldman1995complexity}
Sally~A Goldman and Michael~J Kearns.
\newblock On the complexity of teaching.
\newblock {\em Journal of Computer and System Sciences}, 50(1):20--31, 1995.

\bibitem[\protect\citeauthoryear{Haug \bgroup \em et al.\egroup
  }{2018}]{DBLP:conf/nips/HaugTS18}
Luis Haug, Sebastian Tschiatschek, and Adish Singla.
\newblock Teaching inverse reinforcement learners via features and
  demonstrations.
\newblock In {\em NeurIPS}, 2018.

\bibitem[\protect\citeauthoryear{Hunziker \bgroup \em et al.\egroup
  }{2019}]{hunziker2019teaching}
Anette Hunziker, Yuxin Chen, Oisin Mac~Aodha, Manuel~Gomez Rodriguez, Andreas
  Krause, Pietro Perona, Yisong Yue, and Adish Singla.
\newblock Teaching multiple concepts to a forgetful learner.
\newblock In {\em NeurIPS}, 2019.

\bibitem[\protect\citeauthoryear{Kamalaruban \bgroup \em et al.\egroup
  }{2019}]{DBLP:conf/ijcai/KamalarubanDCS19}
Parameswaran Kamalaruban, Rati Devidze, Volkan Cevher, and Adish Singla.
\newblock Interactive teaching algorithms for inverse reinforcement learning.
\newblock In {\em IJCAI}, pages 2692--2700, 2019.

\bibitem[\protect\citeauthoryear{Khajah \bgroup \em et al.\egroup
  }{2016}]{DBLP:conf/edm/KhajahLM16}
Mohammad Khajah, Robert~V. Lindsey, and Michael Mozer.
\newblock How deep is knowledge tracing?
\newblock In {\em EDM}, 2016.

\bibitem[\protect\citeauthoryear{Klingler \bgroup \em et al.\egroup
  }{2015}]{klingler2015performance}
Severin Klingler, Tanja K{\"a}ser, Barbara Solenthaler, and Markus Gross.
\newblock On the performance characteristics of latent-factor and knowledge
  tracing models.
\newblock {\em International Educational Data Mining Society}, 2015.

\bibitem[\protect\citeauthoryear{Liu \bgroup \em et al.\egroup
  }{2018}]{DBLP:conf/icml/LiuDLLRS18}
Weiyang Liu, Bo~Dai, Xingguo Li, Zhen Liu, James~M. Rehg, and Le~Song.
\newblock Towards black-box iterative machine teaching.
\newblock In {\em ICML}, 2018.

\bibitem[\protect\citeauthoryear{Ma \bgroup \em et al.\egroup
  }{2019}]{DBLP:conf/ijcai/Ma0H19}
Yuzhe Ma, Xiaojin Zhu, and Justin Hsu.
\newblock Data poisoning against differentially-private learners: Attacks and
  defenses.
\newblock In {\em IJCAI}, pages 4732--4738, 2019.

\bibitem[\protect\citeauthoryear{Mac~Aodha \bgroup \em et al.\egroup
  }{2018}]{mac2018teaching}
Oisin Mac~Aodha, Shihan Su, Yuxin Chen, Pietro Perona, and Yisong Yue.
\newblock Teaching categories to human learners with visual explanations.
\newblock In {\em CVPR}, pages 3820--3828, 2018.

\bibitem[\protect\citeauthoryear{Mansouri \bgroup \em et al.\egroup
  }{2019}]{mansouri2019preference}
Farnam Mansouri, Yuxin Chen, Ara Vartanian, Jerry Zhu, and Adish Singla.
\newblock Preference-based batch and sequential teaching: Towards a unified
  view of models.
\newblock In {\em NeurIPS}, pages 9195--9205, 2019.

\bibitem[\protect\citeauthoryear{Mei and Zhu}{2015}]{mei2015using}
Shike Mei and Xiaojin Zhu.
\newblock Using machine teaching to identify optimal training-set attacks on
  machine learners.
\newblock In {\em AAAI}, pages 2871--2877, 2015.

\bibitem[\protect\citeauthoryear{Melo \bgroup \em et al.\egroup
  }{2018}]{DBLP:conf/ijcai/MeloGL18}
Francisco~S. Melo, Carla Guerra, and Manuel Lopes.
\newblock Interactive optimal teaching with unknown learners.
\newblock In {\em IJCAI}, pages 2567--2573, 2018.

\bibitem[\protect\citeauthoryear{Piech \bgroup \em et al.\egroup
  }{2015}]{piech2015deep}
Chris Piech, Jonathan Bassen, Jonathan Huang, Surya Ganguli, Mehran Sahami,
  Leonidas~J Guibas, and Jascha Sohl-Dickstein.
\newblock Deep knowledge tracing.
\newblock In {\em Advances in neural information processing systems}, 2015.

\bibitem[\protect\citeauthoryear{Rafferty \bgroup \em et al.\egroup
  }{2016}]{rafferty2016faster}
Anna~N Rafferty, Emma Brunskill, Thomas~L Griffiths, and Patrick Shafto.
\newblock Faster teaching via pomdp planning.
\newblock {\em Cognitive science}, 2016.

\bibitem[\protect\citeauthoryear{Sen \bgroup \em et al.\egroup
  }{2018}]{sen2018machine}
Ayon Sen, Purav Patel, Martina~A. Rau, Blake Mason, Robert Nowak, Timothy~T.
  Rogers, and Xiaojin Zhu.
\newblock Machine beats human at sequencing visuals for perceptual-fluency
  practice.
\newblock In {\em EDM}, 2018.

\bibitem[\protect\citeauthoryear{Settles and
  Meeder}{2016}]{settles2016trainable}
Burr Settles and Brendan Meeder.
\newblock A trainable spaced repetition model for language learning.
\newblock In {\em ACL}, pages 1848--1858, 2016.

\bibitem[\protect\citeauthoryear{Singla \bgroup \em et al.\egroup
  }{2013}]{singla2013actively}
Adish Singla, Ilija Bogunovic, G~Bart{\'o}k, A~Karbasi, and A~Krause.
\newblock On actively teaching the crowd to classify.
\newblock In {\em NIPS Workshop on Data Driven Education}, 2013.

\bibitem[\protect\citeauthoryear{Singla \bgroup \em et al.\egroup
  }{2014}]{singla2014near}
Adish Singla, Ilija Bogunovic, G{\'a}bor Bart{\'o}k, Amin Karbasi, and Andreas
  Krause.
\newblock Near-optimally teaching the crowd to classify.
\newblock In {\em ICML}, 2014.

\bibitem[\protect\citeauthoryear{Sullivan \bgroup \em et al.\egroup
  }{2009}]{sullivan2009ebird}
Brian Sullivan, Christopher Wood, Marshall Iliff, Rick Bonney, Daniel Fink, and
  Steve Kelling.
\newblock e{B}ird: A citizen-based bird observation network in the biological
  sciences.
\newblock {\em Biological Conservation}, 2009.

\bibitem[\protect\citeauthoryear{Tschiatschek \bgroup \em et al.\egroup
  }{2019}]{DBLP:conf/nips/TschiatschekGHD19}
Sebastian Tschiatschek, Ahana Ghosh, Luis Haug, Rati Devidze, and Adish Singla.
\newblock Learner-aware teaching: Inverse reinforcement learning with
  preferences and constraints.
\newblock In {\em NeurIPS}, 2019.

\bibitem[\protect\citeauthoryear{Van~Horn \bgroup \em et al.\egroup
  }{2018}]{van2018inaturalist}
Grant Van~Horn, Oisin Mac~Aodha, Yang Song, Yin Cui, Chen Sun, Alex Shepard,
  Hartwig Adam, Pietro Perona, and Serge Belongie.
\newblock The inaturalist species classification and detection dataset.
\newblock In {\em CVPR}, 2018.

\bibitem[\protect\citeauthoryear{Yeo \bgroup \em et al.\egroup
  }{2019}]{DBLP:conf/aaai/YeoKSMAFDC19}
Teresa Yeo, Parameswaran Kamalaruban, Adish Singla, Arpit Merchant, Thibault
  Asselborn, Louis Faucon, Pierre Dillenbourg, and Volkan Cevher.
\newblock Iterative classroom teaching.
\newblock In {\em AAAI}, pages 5684--5692, 2019.

\bibitem[\protect\citeauthoryear{Zhang \bgroup \em et al.\egroup
  }{2018}]{DBLP:conf/aaai/Zhang0W18}
Xuezhou Zhang, Xiaojin Zhu, and Stephen~J. Wright.
\newblock Training set debugging using trusted items.
\newblock In {\em AAAI}, pages 4482--4489, 2018.

\bibitem[\protect\citeauthoryear{Zhu \bgroup \em et al.\egroup
  }{2018}]{DBLP:journals/corr/ZhuSingla18}
Xiaojin Zhu, Adish Singla, Sandra Zilles, and Anna~N. Rafferty.
\newblock An overview of machine teaching.
\newblock {\em CoRR}, abs/1801.05927, 2018.

\bibitem[\protect\citeauthoryear{Zhu}{2015}]{zhu2015machine}
Xiaojin Zhu.
\newblock Machine teaching: An inverse problem to machine learning and an
  approach toward optimal education.
\newblock In {\em AAAI}, pages 4083--4087, 2015.

\bibitem[\protect\citeauthoryear{Zhu}{2018}]{zhu2018optimal}
Xiaojin Zhu.
\newblock An optimal control view of adversarial machine learning.
\newblock {\em CoRR}, abs/1811.04422, 2018.

\bibitem[\protect\citeauthoryear{Zilles \bgroup \em et al.\egroup
  }{2011}]{zilles2011models}
Sandra Zilles, Steffen Lange, Robert Holte, and Martin Zinkevich.
\newblock Models of cooperative teaching and learning.
\newblock {\em JMLR}, 12(Feb):349--384, 2011.

\end{thebibliography}
